%% file: main.tex
\documentclass[lettersize,journal]{IEEEtran}
\usepackage{colortbl} 
\usepackage{color}
\usepackage{xcolor}
\usepackage{cite}
\usepackage{hyperref}
\usepackage{graphicx} 
\usepackage{wrapfig}
\usepackage{amsmath,amssymb,amsthm}

\usepackage[utf8]{inputenc} 
\usepackage[T1]{fontenc}    
\usepackage{url}            
\usepackage{booktabs}       
\usepackage{amsfonts}       
\usepackage{nicefrac}       
\usepackage{microtype}      
\usepackage{xspace}
\usepackage{url}
\usepackage{fontawesome5}
\usepackage{mdframed}
\usepackage{tikz}
\usepackage{multirow}
\RequirePackage[export]{adjustbox}
\usepackage{pifont}
\usepackage{cleveref}
\usepackage{float}

\newcommand{\diff}{\mathrm{d}}
\newcommand{\E}{\mathbb{E}}
\newcommand{\Loss}{\mathcal{L}}
\newcommand{\bx}{\mathbf{x}}
\newcommand{\bm}{\boldsymbol{w}_t}
\newcommand{\Real}{\mathbb{R}}
\newcommand{\Normal}{\mathcal{N}}
\newcommand{\Identity}{\boldsymbol{I}}
\newcommand{\rbm}{\bar{\boldsymbol{w}}_t}
\newcommand{\nablaxx}{\nabla_{\hspace{-0.5mm}\bx}}

\newcommand{\noise}{\boldsymbol{\epsilon}}
\newcommand{\by}{\mathbf{y}}

\newcommand{\bU}{\mathbf{U}}
\newcommand{\bV}{\mathbf{V}}
\newcommand{\bX}{\mathbf{X}}

\newcommand{\btheta}{\mathbf{\theta}}
\newcommand{\bphi}{\mathbf{\phi}}
\newtheorem{theorem}{Theorem}

\newtheorem{lemma}{Lemma}
\newtheorem{remark}{Remark}

\DeclareMathOperator*{\argmin}{argmin}
\DeclareMathOperator*{\Tr}{Tr}

\DeclareMathOperator*{\Cov}{Cov}

\definecolor{mark}{RGB}{214, 235, 245}
\newcommand{\mk}{\cellcolor{mark}}
\newcommand{\themodel}{\textbf{\texttt{Graffe}}\xspace}
\makeatletter
\patchcmd{\@makecaption}
  {\scshape}
  {}
  {}
  {}
\makeatletter
\patchcmd{\@makecaption}
  {\\}
  {.\ }
  {}
  {}
\makeatother
\begin{document}

\title{Graffe: Graph Representation Learning via Diffusion Probabilistic Models}

\author{Dingshuo Chen$^*$, 
        Shuchen Xue$^*$,
        Liuji Chen, 
        Yingheng Wang, 
        Qiang Liu, \IEEEmembership{Member, IEEE}, \\
        Shu Wu$^\dagger$, \IEEEmembership{Senior Member, IEEE},
        Zhi-Ming Ma,
        and Liang Wang, \IEEEmembership{Fellow, IEEE}
\thanks{$^*$~Equal contribution (alphabetical order)}        
\thanks{$^\dagger$~Corresponding author}

\thanks{Dingshuo Chen, Liuji Chen, Qiang Liu,  Shu Wu, and Liang Wang are with the State Key Laboratory of Multimodal Artificial Intelligence Systems (MAIS), Institute of Automation, Chinese Academy of Sciences (CASIA), Beijing 100190, China}

\thanks{Shuchen Xue and Zhi-Ming Ma are
with the Academy of Mathematics and Systems Science, Chinese Academy of Sciences (CAS), Beijing 100190, China. 
}

\thanks{Yingheng Wang is with the Department of Computer Science, Cornell University, Ithaca, NY 14853, USA. 
} }

\markboth{Journal of \LaTeX\ Class Files,~Vol.~14, No.~8, August~2021}%
{Shell \MakeLowercase{\textit{et al.}}: A Sample Article Using IEEEtran.cls for IEEE Journals}


\maketitle
\input{sections/0_abs}
\input{sections/1_intro}

\input{sections/2_related_work}
\input{sections/3_preliminary}

\input{sections/4_analysis}
\input{sections/5_method}

\input{sections/6_exp}
\input{sections/7_conclusion}

\bibliographystyle{IEEEtran}
\bibliography{references}

\input{sections/appendix}



\vfill

\end{document}

%% file: sections/0_abs.tex
\begin{abstract}
Diffusion probabilistic models (DPMs), widely recognized for their potential to generate high-quality samples, tend to go unnoticed in representation learning. While recent progress has highlighted their potential for capturing visual semantics, adapting DPMs to graph representation learning remains in its infancy. In this paper, we introduce \themodel, a self-supervised diffusion model proposed for graph representation learning. It features a graph encoder that distills a source graph into a compact representation, which, in turn, serves as the condition to guide the denoising process of the diffusion decoder. To evaluate the effectiveness of our model, we first explore the theoretical foundations of applying diffusion models to representation learning, proving that the denoising objective implicitly maximizes the conditional mutual information between data and its representation. Specifically, we prove that the negative logarithm of the denoising score matching loss is a tractable lower bound for the conditional mutual information. Empirically, we conduct a series of case studies to validate our theoretical insights. In addition, \themodel delivers competitive results under the linear probing setting on node and graph classification tasks, achieving state-of-the-art performance on 9 of the 11 real-world datasets. These findings indicate that powerful generative models, especially diffusion models, serve as an effective tool for graph representation learning.
\end{abstract}


%% file: sections/1_intro.tex
\section{Introduction}
\label{sec:intro}
Self-supervised learning (SSL), which enables effective data understanding without laborious human annotations, is emerging as a key paradigm for addressing both generative and discriminative tasks. When we revisit the evolution of SSL across these two tasks, interestingly, a mutually reinforcing manner becomes evident: 
Progress in one aspect often stimulates progress in the other.
For instance, autoencoder \cite{hinton2006reducing}, which initially made a mark in feature extraction, laid the foundation for the success of VAEs \cite{kingma2013auto} for sample generation. Conversely, breakthroughs in generative tasks like autoregression \cite{radford2018improving} and adversarial training \cite{goodfellow2020generative}, have deepened our understanding of representation learning, driving the development of iGPT \cite{chen2020generative} and BigBiGAN \cite{donahue2019large}. 

Recently, diffusion models\cite{ho2020denoising,song2020score} have demonstrated astonishing generation quality in different domains, particularly in terms of realism, detail depiction, and distribution coverage. A natural question arises: \emph{can we draw on the successful experiences of diffusion models to enhance representation learning?} This issue is particularly pressing in the context of graph learning, since generation—the ability to create—plays a less critical role compared to discrimination on graphs, e.g., social networks, citation networks, and recommendation networks. The question seems not difficult to address, as generation is considered one of the highest manifestations of learning thus having powerful capability to learn high-quality representation \cite{krathwohl2002revision, johnson2018image, wang2023infodiffusion, hudson2024soda}; however, the reality is much more complex. 

To generalize the representation learning power of diffusion models on graph data, two main impediments must be addressed: \ding{172} \textbf{the non-Euclidean nature of graph data}, which complicates the direct application of diffusion models and necessitates consideration of both structural and feature information \cite{li2023survey,li2023gslb}; \ding{173} \textbf{the absence of an encoder component in diffusion model} prevents us from obtaining explicit data representation and finetuning encoder in downstream tasks. Motivated to overcome these challenges, we investigate how to adapt diffusion models to graph representation learning and enhance their discrimination performance.

\begin{figure*}[t]
    \centering    \includegraphics[width=\textwidth,height=!]{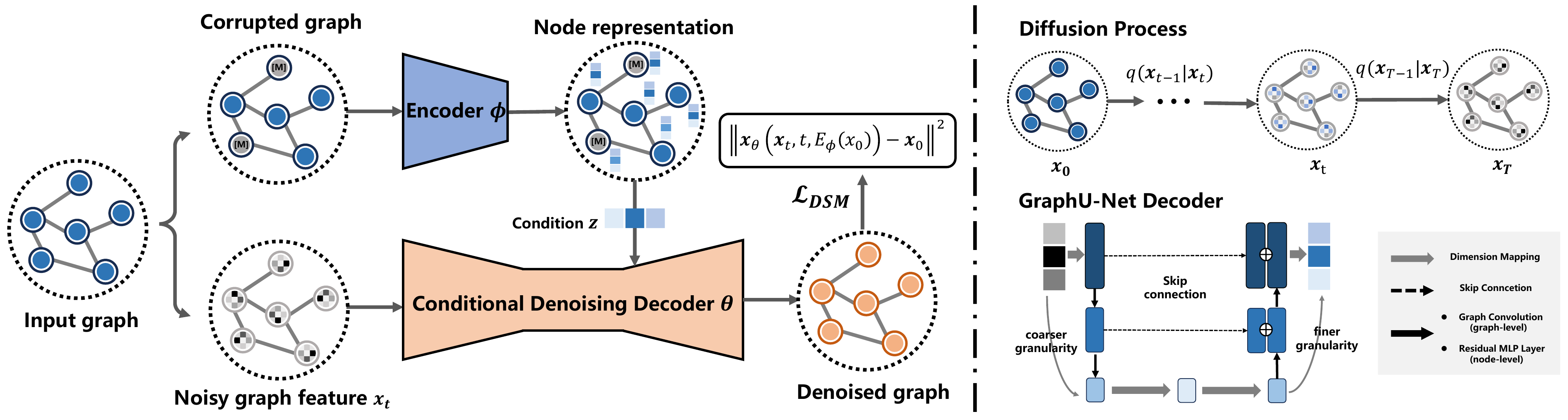}
    \caption{The overall framework of \themodel. \textbf{(Left)} The input graph has certain nodes corrupted and is subsequently fed into a GNN encoder to obtain node representations as the condition. The decoder then receives both the noisy graph features $\bx_t$ and the condition $\mathbf{z}$ as inputs to perform denoising, aiming to restore the original node features $\bx_0$. \textbf{(Right)} The diffusion process of graph features and the architecture of GraphU-Net decoder.}
    \label{fig:framework}
\end{figure*}
This work is particularly relevant to approaches that use diffusion models to capture high-level semantics for classification tasks while enhancing representational capacity. Those approaches can be broadly categorized into two main groups: \emph{(i)} one treats part of the diffusion model itself as a feature extractor (\underline{\emph{implicit-encoder pattern}})\cite{xiang2023denoising, chen2024deconstructing, yang2024directional}. They obtain the latent representation from a certain intermediate layer, which inevitably exposes them to challenge \ding{173}. \emph{(ii)} Another line of work jointly trains the diffusion model and an additional feature extractor  (\underline{\emph{explicit-encoder pattern}})\cite{abstreiter2021diffusion, wang2023infodiffusion, hudson2024soda}. However, the latter pattern have struggled to surpass their contrastive and auto-encoding counterparts.

In this paper, we propose \themodel, which shares a philosophy similar to the explicit-encoder pattern. Starting with the optimization objective for diffusion-based SSL, we analyze diffusion representation learning (DRL) and show that it maximizes the mutual information lower bound between the learned representation and the original input, with more informative representations leading to lower denoising score matching loss, and vice versa. This suggests that DRL implicitly follows a principle akin to the InfoMax principle \cite{linsker1988self, hjelm2018learning}, which we call the Diff-InfoMax principle. Furthermore, we observe from the frequency domain of graph features that DRL excels in capturing high-frequency information. Inspired by our theoretical insights, we instantiate our model with a graph neural network (GNN) encoder for explicit representation extraction and a tailored diffusion decoder, both trained from scratch in tandem. The encoder transforms the graph structure and feature information into a compact representation, which acts as a condition for the decoder together with noisy features to guide the denoising process. The main contributions of this work are three-fold:

\ding{182} We theoretically prove that the negative logarithm of the denoising score matching loss is a tractable lower bound for conditional mutual information. Building on this, we introduce the Diff-InfoMax principle, an extension of the standard InfoMax principle, showing that DRL implicitly follows it.

\ding{183} We propose an effective diffusion-based representation learning method catering to graph tasks, termed as \themodel. Equipped with random node masking and customized diffusion architecture for different task types, it can achieve sufficient graph understanding and obtain representations with rich semantic information.

\ding{184} We conduct extensive experiments on 11 classification tasks under the linear protocol, spanning node- and graph-level tasks of diverse domains. Our method can achieve state-of-the-art or near-optimal performance across all datasets. On \texttt{Computer}, \texttt{Photo}, and \texttt{COLLAB} datasets, our model set a new accuracy record of 91.3\%, 94.2\% and 81.3\%, respectively.

%% file: sections/2_related_work.tex
\section{Related Work}

\subsection{Self-supervised Learning on Graphs}
\paragraph{Contrastive methods} Being popular in SSL, contrastive methods aim to learn discriminative representations by contrasting positive and negative samples. 
The key to obtain distinguishable representations lies in the way of constructing contrastive pairs. DGI \cite{velickovic2019deep} and InfoGraph \cite{sun2019infograph}, based on MI maximization, corrupt graph feature and topology to construct negative samples. To avoid the underlying risk of semantic damage, GRACE \cite{zhu2020deep}, GCA \cite{zhu2021graph}, and GraphCL \cite{you2020graph} use other graphs within the same batch as negatives. This approach helps to mitigate issues related to graph-specific distortions while still maintaining the contrastive nature of the objective. Other works, i.e., BGRL \cite{thakoor2021large} and CCA-SSA \cite{zhang2021canonical}, propose to achieve contrastive learning free of negatives yet demanding strong regularization or feature decorrelation. A line of works borrow from data augmentation in the field of computer vision (CV) to construct constrastive pairs, including feature-oriented (\cite{thakoor2021large,you2020graph,zhu2020deep}, shuffling \cite{velickovic2019deep}), perturbation \cite{hu2020gpt,you2020graph}), and graph-theory-based (random walk \cite{hassani2020contrastive,qiu2020gcc}.

\paragraph{Generative methods} Generative self-supervised methods aim to learn informative representations using learning signals from the data itself, usually by maximizing the marginal log-likelihood of the data. GPT-GNN \cite{hu2020gpt}, following the auto-regressive paradigm, iteratively generates graph features and topology, which is unnatural as most graph data has no inherent order. GAE and VGAE \cite{kipf2016variational} learn to reconstruct the adjacency matrix by using the representation learned from GCN, while other graph autoencoders \cite{salehi2019graph,hou2022graphmae, chen2023uncovering, chen2024beyond, zhang2024gder} further combine it with feature reconstruction with tailored strategies. However, these generative methods are usually not principled in terms of probabilistic generative models and often prove to be inferior to the contrastive ones. The reliance on reconstruction-based objectives often limits the ability of these models to capture more complex, higher-level relationships in the graph data.

\subsection{Diffusion Models for Representation Learning}
The very first attempt has combined auto-encoders with diffusion models---e.g., DiffAE \cite{preechakul2022diffusion}, a non-probabilistic auto-encoder model that produces semantically meaningful latent. InfoDiffusion \cite{wang2023infodiffusion}, as the first principled probabilistic generative model for representation learning, augments DiffAE with an auxiliary-variable model family and mutual information maximization. Similarly, \cite{zhang2022unsupervised} uses a pre-trained diffusion decoder and designs a re-weighting scheme to fill in the posterior mean gap. Targeting image classification tasks,  \cite{wei2023diffusion, hudson2024soda} combine latent diffusion with the self-supervised learning objective to get meaningful representations. The decoder-only models \cite{xiang2023denoising, chen2024deconstructing}, directly use the representations from intermediate layers without auxiliary encoders. However, the use of expressive diffusion models for graph representation learning remains under-explored. DDM \cite{yang2024directional} takes an initial step, but the proposed diffusion process is not mathematically rigorous and principled.

%% file: sections/3_preliminary.tex
\section{Preliminary}

\subsection{Backgound on Diffusion Model}
Diffusion Probabilistic Models (DPMs) construct noisy data through the stochastic differential equation (SDE):
\begin{equation}
\label{eq: forward process}
\diff \bx_t = f(t) \bx_t \diff t + g(t)  \diff \bm,
\end{equation}
where $f(t),g(t):\Real \rightarrow \Real$ is scalar functions such that for each time $t \in [0,T]$, $\bx_t | \bx_0 \sim \Normal(\alpha_t \bx_0, \sigma^2_t\Identity)$, $\alpha_t$, $\sigma_t$ are determined by $f(t)$, $g(t)$, $\bm \in \Real^d$ represents the standard Wiener process. It was demonstrated in\cite{anderson1982reverse} that the forward process~\eqref{eq: forward process} has an equivalent reverse-time diffusion process (from $T$ to $0$), allowing the generation process to be equivalent to numerically solving the reverse SDE\cite{ho2020denoising,song2020score,lu2022dpm,xue2023sa,xue2024accelerating}.   
\begin{equation}
\label{eq: reverse SDE}
\diff \bx_t = \left[f(t) \bx_t - g^2(t) \nablaxx \log p_t(\bx_t)\right] \diff t + g(t)  \diff \rbm,
\end{equation}
where $\rbm$ represents the Wiener process in reverse time, and $\nablaxx\log p_t(\bx)$ is the score function. To get the \textit{score function} $\nablaxx\log p_t(\bx_t)$ in \eqref{eq: reverse SDE}, we usually take neural network $\boldsymbol{s}_{\boldsymbol{\theta}}(\bx, t)$ parameterized by $\boldsymbol{\theta}$ to approximate it by optimizing the Denoising Score Matching loss\cite{song2020score}: 
\begin{equation}
\label{eq: score matching loss}
\E_t \left\{ \Tilde{\lambda}(t) \E_{\bx_0,\bx_t} \left[ \left\| \boldsymbol{s}_{\boldsymbol{\theta}}(\bx, t) - \nabla_{\bx_t}\log p(\bx_t|\bx_0) \right\|_2^2 \right] \right\},
\end{equation}
where $\Tilde{\lambda}(t)$ is a loss weighting function over time. In practice, several methods are used to reparameterize the score-based model. The most popular approach\cite{ho2020denoising} utilizes a \textit{noise prediction model} such that  $\noise_{\boldsymbol{\theta}} (\bx_t, t) = -\sigma_t \boldsymbol{s}_{\boldsymbol{\theta}}(\bx_t, t)$, while others employ a \textit{data prediction model}, represented by $\bx_{\boldsymbol{\theta}} (\bx_t, t) = (\bx_t - \sigma_t \noise_{\boldsymbol{\theta}} (\bx_t, t))/\alpha_t$. The DSM loss is equivalent to the following data prediction loss after changing the weighting function:
\begin{equation}
\Loss_{\bx_0, DSM} = \E_t\left\{\lambda(t) \E_{\bx_0} \E_{\bx_t | \bx_0} \left[ \left\| \bx_{\btheta}(\bx_t, t) - \bx_0 \right\|^2 \right] \right\}.
\end{equation}
\subsection{InfoMax Principle}

Unsupervised representation learning is a key challenge in machine learning, and recently, there has been a resurgence of methods motivated by the InfoMax principle\cite{hjelm2018learning}. Mutual Information (MI) quantifies the "amount of information" obtained about one random variable $X$ by observing the other random variable $Y$. Formally, the MI between $X$ and $Y$ with joint density $p(x,y)$ and marginal densities $p(x)$ and $p(y)$, is defined as the Kullback-Leibler divergence between the joint distribution and the product of the marginal distribution
\begin{equation}
\begin{split}
I(X;Y) &= D_{KL}(P_{(X,Y)} \| P_X \otimes P_Y)\\
&= \E_{p(x,y)} \left[ \log\frac{p(x,y)}{p(x)p(y)} \right].    
\end{split}
\end{equation}
The InfoMax principle chooses a representation $f(x)$ by maximizing the mutual information between the input $x$ and the representation $f(x)$. However, estimating MI, especially in high-dimensional spaces is challenging in nature. And one often optimizes a tractable lower bound of MI in practice\cite{poole2019variational}.



%% file: sections/4_analysis.tex
\section{An Information-Theoretic Perspective on Diffusion Representation Learning}

Despite some empirical attempts at Diffusion Representation Learning (DRL), its theoretical foundations remain largely uncharted. In this section, we analyze the DRL through the lens of Information Theory, establishing a connection between the DRL objective and mutual information.

\subsection{The Role of Extra Information in Improving Reconstruction}
\label{sec:extrac_info}

Conditional diffusion models exhibit superior generation quality and lower denoising score matching loss than their unconditional counterparts, as observed by\cite{dhariwal2021diffusion, zhang2022unsupervised}. Figure~\ref{fig:loss_comp} illustrates the denoising score matching loss for the label conditional task (\textbf{Label} curve) is lower than that for the unconditional task (\textbf{Vanilla} curve). This improvement is attributed to the additional information provided by class labels, which aids the diffusion model in effectively denoising noisy data. One might consider class labels $c$ as a special feature extracted from data: $c = E_{\bphi}(\bx)$ where $E_{\bphi}$ is a classifier that outputs class labels. This leads to speculation that more informative representations further enhance the denoising process and lower the denoising score matching loss conditioned on the representations. Thus intuitively one can jointly train the diffusion model conditioning on an additional feature extractor $E_{\bphi}$\cite{abstreiter2021diffusion, hudson2024soda}, as the reconstruction denoising loss will guide the feature extractor toward more informative representations. Formally, the learning objective for DRL is as follows:
\begin{equation}
\begin{split}
&\Loss_{\bx_0, DSM, \bphi}\\
= &\E_t\left\{\lambda(t)\E_{\bx_0} \E_{\bx_t | \bx_0} \left[ \left\| \bx_{\btheta}(\bx_t, t, E_{\bphi}(\bx_0)) - \bx_0 \right\|^2 \right]\right\}.
\end{split}
\end{equation}
In the next part of this section, we elucidate the intuition that more informative representations lead to lower denoising score matching loss from a theoretical standpoint. We eliminate the effects of limited network capacity or optimization errors, allowing us to investigate the influence of additional conditions on the denoising score matching loss under ideal conditions—specifically when the network capacity is adequate and optimization achieves its optimal state. The following theorem demonstrates that the denoising score matching objective has a positive lower bound, even when the network's capacity is sufficiently large.

\begin{theorem}\label{thm:analytic loss min}
The denoising score matching objective $\Loss_{\bx_0, DSM}$ has a \textbf{strictly positive} lower bound, regardless of the network capacity and expressive power
\begin{equation}
\begin{aligned}
&\min_{\bx_{\theta}} \Loss_{\bx_0, DSM} \\
= &\min_{\bx_{\theta}}\E_t\left\{\lambda(t) \E_{\bx_0} \E_{\bx_t | \bx_0} \left[ \left\| \bx_{\btheta}(\bx_t, t) - \bx_0 \right\|^2 \right] \right\}\\
= &\E_t\left\{\lambda(t)\E_{\bx_t}\left[\Tr(\Cov[\bx_0 | \bx_t])\right]\right\} > 0,
\end{aligned}
\end{equation}
where $\Tr$ is the Trace of matrix and $\Cov$ is the covariance matrix. The conditioned denoising score matching objective objective $\Loss_{\bx_0, DSM, \bphi}$ has a \textbf{non-negative} lower bound, i.e.
\begin{equation}
\begin{aligned}
&\min_{\bx_{\theta}} \Loss_{\bx_0, DSM, \bphi}\\
=& \E_t\left\{\lambda(t)\E_{\bx_0,\bx_t}\left[\Tr\left(\Cov\left[\bx_0 | \bx_t, E_{\bphi}(\bx_0)\right]\right)\right]\right\} \geq 0.
\end{aligned}
\end{equation}
\end{theorem}
\begin{figure*}[t]
    \centering    \includegraphics[width=0.75\textwidth,height=!]{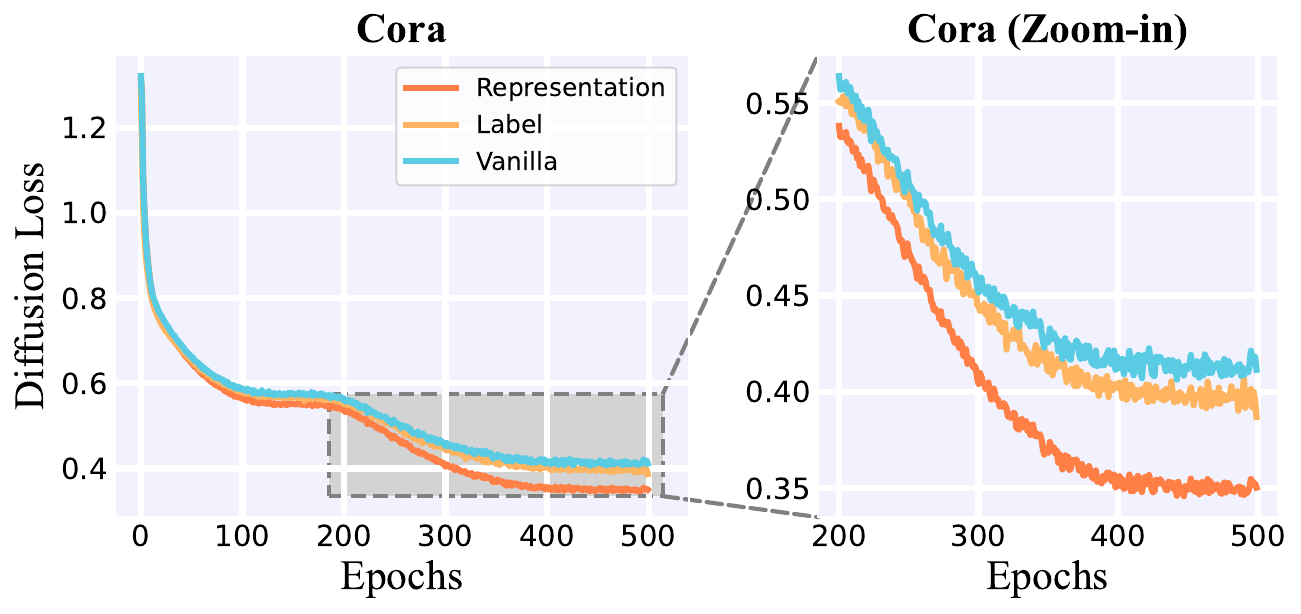}
    \caption{The comparison of denoising losses using different conditions on Cora datasets. \textbf{(Vanilla)} The denoising loss without condition information. \textbf{(Label)} Class label information obtained via linear embedding. \textbf{(Representation)} Learned representations obtained from \themodel.}
    \label{fig:loss_comp}
\end{figure*}
The proof is provided in Appendix A.
Theorem~\ref{thm:analytic loss min} reveals an attractive property of the denoising score matching loss: its minimum value is determined by the uncertainty of the conditional distribution (the trace of the covariance matrix serves as a multidimensional generalization of variance). Additionally, Theorem~\ref{thm:lower loss} demonstrates that the supplementary information provided by the feature extractor $E_{\bphi}$ reduces the lower bound of DSM by decreasing the uncertainty of the conditional distribution through more informative representations.

To formally demonstrate the claim in Theorem~\ref{thm:lower loss} regarding the reduction of the loss lower bound, we rely on two fundamental results concerning conditional expectations, presented below as lemmas. 
The proofs of lemmas are in Appendix A.
\begin{lemma}\label{lemma:prob1}
$\bU$ and $\bV$ are two square-integrable random variables. $\bU$ is $\mathcal{G}$-measurable and $\E\left[\bV|\mathcal{G}\right] = \mathbf{0}$, then
\begin{equation}
\E \left[ \left\| \bU+\bV \right\|^2 \right] = \E \left[ \left\| \bU \right\|^2 \right] + \E \left[ \left\| \bV \right\|^2 \right].
\end{equation}
\end{lemma}
Lemma~\ref{lemma:prob1} establishes an orthogonality condition. This condition allows us to prove the following lemma concerning the effect of increasing information (represented by larger sigma-algebras) on conditional expectations.
\begin{lemma}\label{lemma:prob2}
$\bX$ is a random variable, $\mathcal{F}$ and $\mathcal{G}$ are two $\sigma$-algebras such that $\mathcal{G} \subset \mathcal{F}$, then we have 
\begin{equation}
\E \left[ \left\| \E \left[ \bX \middle| \mathcal{F} \right] \right\|^2 \right] \geq \E \left[ \left\| \E \left[ \bX \middle| \mathcal{G} \right] \right\|^2 \right].
\end{equation}

\end{lemma}
Equipped with Lemma~\ref{lemma:prob1} and Lemma~\ref{lemma:prob2}, we are now prepared to formally state Theorem~\ref{thm:lower loss}, which compares the minimum achievable loss values.
\begin{theorem}\label{thm:lower loss}
The conditioned denoising score matching objective $\Loss_{\bx_0, DSM, \bphi}$ has a smaller minimum compared with the vanilla objective:
\begin{equation}
\min_{\bx_{\theta}} \Loss_{\bx_0, DSM, \bphi} \leq \min_{\bx_{\theta}}\Loss_{\bx_0, DSM}.
\end{equation}    
\end{theorem}
\begin{proof}
According to~\Cref{thm:analytic loss min}, the minimum values for the vanilla and conditioned objectives are known to be:
\begin{equation}
\min_{\bx_{\theta}} \Loss_{\bx_0, DSM} = \E_t\left\{\lambda(t)\E_{\bx_t}\left[\Tr(\Cov[\bx_0 | \bx_t])\right]\right\}.
\end{equation}
\begin{equation}
\min_{\bx_{\theta}} \Loss_{\bx_0, DSM, \bphi} = \E_t\left\{\lambda(t)\E_{\bx_0, \bx_t}\left[\Tr(\Cov[\bx_0 | \bx_t, E_{\bphi}(\bx_0)])\right]\right\}.
\end{equation}
To establish the theorem, it is sufficient to prove the following inequality holds for the terms inside the expectation over $t$:
\begin{equation}
\E_{\bx_0, \bx_t}\left[\Tr(\Cov[\bx_0 | \bx_t, E_{\bphi}(\bx_0)])\right] \leq \E_{\bx_t}\left[\Tr(\Cov[\bx_0 | \bx_t])\right].
\end{equation}
Recall that the trace of the conditional covariance matrix is related to the expected squared error of the conditional mean estimator: $\E_{Y}[\Tr(\Cov[X | Y])] = \E_{X,Y}[\|X - \E[X | Y]\|^2]$. The inequality above is equivalent to showing:
\begin{equation}
\begin{aligned}
&\E_{\bx_0, \bx_t} \left[ \left\|
\E \left[ \bx_0 \middle| \bx_t, E_{\bphi}(\bx_0) \right] - \bx_0 \right\|^2 \right] \\
\leq &\E_{\bx_0, \bx_t} \left[ \left\|
\E \left[ \bx_0 \middle| \bx_t \right] - \bx_0 \right\|^2 \right].
\end{aligned}
\end{equation}
Let us expand the left-hand side term. Using the linearity of expectation and the tower property, we derive:
\begin{equation}
\begin{aligned}
&\E_{\bx_0, \bx_t} \left[ \left\| \E [\bx_0 | \bx_t, E_{\bphi}(\bx_0)] - \bx_0 \right\|^2 \right]\\
= & \E_{\bx_0, \bx_t} \left[ \left\| \E [\bx_0 | \bx_t, E_{\bphi}(\bx_0)] \right\|^2\right] + \E_{\bx_0, \bx_t} \left[ \| \bx_0 \|^2\right] \\
& - \E_{\bx_0, \bx_t} \left[ 2\left\langle \E [\bx_0 | \bx_t, E_{\bphi}(\bx_0)] , \bx_0 \right\rangle \right] \\
= & \E_{\bx_0, \bx_t} \left[ \left\| \E [\bx_0 | \bx_t, E_{\bphi}(\bx_0)] \right\|^2\right] + \E_{\bx_0, \bx_t} \left[ \| \bx_0 \|^2\right] \\
& - \E_{\bx_t, E_{\bphi}(\bx_0)}\E_{\bx_0|\bx_t, E_{\bphi}(\bx_0)} \left[ 2\left\langle \E [\bx_0 | \bx_t, E_{\bphi}(\bx_0)] , \bx_0 \right\rangle \right] \\
= & \E_{\bx_0, \bx_t} \left[ \left\| \E [\bx_0 | \bx_t, E_{\bphi}(\bx_0)] \right\|^2\right] + \E_{\bx_0, \bx_t} \left[ \| \bx_0 \|^2\right] \\
& - 2\E_{\bx_t, E_{\bphi}(\bx_0)}\left[ \left\langle \E [\bx_0 | \bx_t, E_{\bphi}(\bx_0)], \E [\bx_0 | \bx_t, E_{\bphi}(\bx_0)] \right\rangle\right] \\
= & \E_{\bx_0, \bx_t} \left[ \| \bx_0 \|^2\right] - \E_{\bx_0, \bx_t} \left[ \left\| \E [\bx_0 | \bx_t, E_{\bphi}(\bx_0)] \right\|^2\right].
\end{aligned}
\end{equation}
Similarly, for the right-hand side term, we have:
\begin{equation}
\begin{aligned}
&\E_{\bx_0, \bx_t} \left[ \left\| \E \left[ \bx_0 \middle| \bx_t \right] - \bx_0 \right\|^2 \right] \\
= & \E_{\bx_0, \bx_t} \left[ \left\| \bx_0 \right\|^2 \right] - \E_{\bx_0, \bx_t} \left[ \left\| \E \left[ \bx_0 \middle| \bx_t \right] \right\|^2 \right].
\end{aligned}
\end{equation}
Thus it's equivalent to proving the following inequality
\begin{equation}
\E_{\bx_0, \bx_t} \left[ \left\| \E \left[ \bx_0 \middle| \bx_t \right] \right\|^2 \right] 
\leq 
\E_{\bx_0, \bx_t} \left[ \left\| \E \left[ \bx_0 \middle| \bx_t, E_{\bphi}(\bx_0) \right] \right\|^2 \right].
\end{equation}
Note that the $\sigma$-algebra $\sigma(\bx_t) \subset \sigma(\bx_t, E_{\bphi}(\bx_0))$, according to lemma \ref{lemma:prob2}, the result holds.
\end{proof}
Theorem~\ref{thm:lower loss} offers a qualitative insight, indicating that informative representations diminish the uncertainty in the conditional distribution. Figure~\ref{fig:loss_comp} shows the denoising score matching loss for the representation conditional task (\textbf{Representation} curve) is lower than both the unconditional task (\textbf{Vanilla} curve) and the label conditional task (\textbf{Label} curve). This suggests that the learned representation contains richer information than class labels alone.

\subsection{Diff-InfoMax Principle}

Intuitively a poor representation dominated by noise provides little useful information, failing to assist the diffusion model in denoising. In contrast, a rich and informative representation enhances the model's denoising capabilities. In this section, we will quantitatively analyze this from an information-theoretic perspective. Notably, the DRL objective is closely related to the conditional mutual information between $E_{\bphi}(\bx_0)$ and $\bx_0$ given $\bx_t$. Our information-theoretic analysis relies on relating the uncertainty measured by the DSM loss to entropy. The following lemma identifies the distribution that maximizes entropy under constraints relevant to our analysis, namely a fixed trace of the covariance matrix.
\begin{lemma}\label{lemma:max entropy}
Let $\Pi_t$ be the set of distribution $p(x)$ on $\mathbb{R}^n$ satisfying the following condition:
\begin{equation}
    \E_p\left[\mathbf{X}\right] = \mathbf{0}, \quad\Tr\left(\Cov_p\left[\mathbf{X}\right]\right) = t.
\end{equation}
Then the n-dimensional Gaussian distribution with mean $\mathbf{0}$ and covariance matrix $\Sigma = \frac{t}{n}I_n$ is the maximum entropy distribution in $\Pi_t$
\end{lemma}
The proof is provided in Appendix A. 
Leveraging Lemma~\ref{lemma:max entropy}, which bounds the entropy for a given variance (trace), we can now state and prove the theorem linking the DSM loss to conditional mutual information.
\begin{theorem}\label{thm:diff infomax}
Suppose $\bx_0 \in \Real^d$, let  $\Loss_{\bx_0, DSM, \bphi, t} = \E_{\bx_0,\bx_t} \left[ \left\| \bx_{\btheta}(\bx_t, t, E_{\bphi}(\bx_0)) - \bx_0 \right\|^2 \right]$be the conditional denoising score matching loss at time $t$, and let $h(\bx|\by)$ be the conditional entropy of $\bx$ given $\by$, then the negative logarithm of denoising score matching loss is a lower bound for the conditional mutual information between data and feature, which quantifies the shared information between $\bx_0$ and $E_{\bphi}(\bx_0)$, given the knowledge of $\bx_t$
\begin{equation}
\begin{aligned}
&\quad I\left(\bx_0 \,;\, E_{\bphi}(\bx_0) \mid \bx_t\right) \geq - \log \Loss_{\bx_0, DSM, \bphi, t} + C, \\
&\text{where } C = \log \left( \frac{d}{2\pi e} \right) + \frac{2}{d} \, h\left(\bx_0 \mid \bx_t\right) \text{ is a constant.}
\end{aligned}
\end{equation}
\end{theorem}
\begin{proof}
The proof begins by applying Lemma \ref{lemma:max entropy}, which relates conditional entropy to the trace of the conditional covariance matrix.
\begin{equation}
\begin{aligned}
&h(\bx_0 \mid \bx_t = \bx, E_{\bphi}(\bx_0) = \by) \\ 
\leq& \frac{d}{2} \left( 1 + \log \left( \frac{2\pi \Tr\left(\Cov\left[\bx_0 \mid \bx_t = \bx, E_{\bphi}(\bx_0) = \by\right]\right)}{d} \right) \right).\\
&\Tr\left(\Cov\left[\bx_0 \mid \bx_t = \bx, \, E_{\bphi}(\bx_0) = \by\right]\right) \\ 
\geq &\frac{d}{2 \pi e} \exp{\left(\frac{2h\left(\bx_0 \mid \bx_t = \bx, \, E_{\bphi}(\bx_0) = \by\right)}{d}\right)}.
\end{aligned}
\end{equation}
Taking the expectation over $\bx_0$ and $\bx_t$ on both sides of the trace inequality, and applying Jensen's inequality to the right-hand side (since the exponential function is convex), we obtain:
\begin{equation}
\begin{aligned}
&\E_{\bx_0, \bx_t}\left[\Tr\left(\Cov\left[\bx_0 \mid \bx_t, \, E_{\bphi}(\bx_0)\right]\right)\right]\\ 
\geq &\frac{d}{2 \pi e} \exp{\left(\frac{2h\left(\bx_0 \mid \bx_t, \, E_{\bphi}(\bx_0)\right)}{d}\right)}.
\end{aligned}
\end{equation}
This inequality can be rearranged to yield an upper bound for the conditional entropy $h(\bx_0 \mid \bx_t, E_{\bphi}(\bx_0))$:
\begin{equation}
\begin{aligned}
&h\left(\bx_0 \mid \bx_t, \, E_{\bphi}(\bx_0)\right)\\
\leq &\frac{d}{2}\,\log{\left(\frac{2\pi e}{d} \E_{\bx_0, \bx_t}\left[\Tr\left(\Cov\left[\bx_0 \mid \bx_t, \, E_{\bphi}(\bx_0)\right]\right)\right]\right)}.
\end{aligned}
\end{equation}
We arrive at the following lower bound for the mutual information:
\begin{equation}
\begin{aligned}
&I\left(\bx_0 \,;\, \bx_t, \, E_{\bphi}(\bx_0)\right) \\ 
=& h(\bx_0) - h\left(\bx_0 \mid \bx_t, \, E_{\bphi}(\bx_0)\right) \\ 
\geq& h(\bx_0) - \frac{d}{2}\,\log{\left(\frac{2\pi e}{d} \E_{\bx_0, \bx_t}\left[\Tr\left(\Cov\left[\bx_0 \mid \bx_t, \, E_{\bphi}(\bx_0)\right]\right)\right]\right)}. 
\end{aligned}
\end{equation}
We now apply the chain rule for mutual information:
\begin{equation}
I\left(\bx_0 \,;\, \bx_t, \, E_{\bphi}(\bx_0)\right) = I\left(\bx_0 \,;\, \bx_t\right) + I\left(\bx_0 \,;\, E_{\bphi}(\bx_0) \mid \bx_t\right).
\end{equation}
Substituting the chain rule,
\begin{equation}
\begin{aligned}
&\frac{d}{2}\,\log{\left(\frac{2\pi e}{d} \E_{\bx_0, \bx_t}\left[\Tr\left(\Cov\left[\bx_0 \mid \bx_t, \, E_{\bphi}(\bx_0)\right]\right)\right]\right)}\\
\geq &h(\bx_0) - I\left(\bx_0 \,;\, \bx_t\right) - I\left(\bx_0 \,;\, E_{\bphi}(\bx_0) \mid \bx_t\right)\\
\geq &h\left(\bx_0 \mid \bx_t\right) - I\left(\bx_0 \,;\, E_{\bphi}(\bx_0) \mid \bx_t\right).
\end{aligned}
\end{equation}
Exponentiating both sides and rearranging establishes the following lower bound on the expected trace term:
\begin{equation}
\begin{aligned}
&\E_{\bx_0, \bx_t}\left[\Tr\left(\Cov\left[\bx_0 \mid \bx_t, \, E_{\bphi}(\bx_0)\right]\right)\right]\\
\geq &\frac{d}{2\pi e} \exp{\left(\frac{2}{d} h\left(\bx_0 \mid \bx_t\right)\right)} \, \exp{\left(- I\left(\bx_0 \,;\, E_{\bphi}(\bx_0) \mid \bx_t\right) \right)}.
\end{aligned}
\end{equation}
The loss is always lower-bounded according to~\Cref{thm:lower loss}:
\begin{equation}
\Loss_{\bx_0, DSM, \bphi, t} \geq \E_{\bx_0, \bx_t}\left[\Tr\left(\Cov\left[\bx_0 \mid \bx_t, \, E_{\bphi}(\bx_0)\right]\right)\right].
\end{equation}
Thus
\begin{equation}
\begin{aligned}
&\Loss_{\bx_0, DSM, \bphi, t}\\
\geq &\frac{d}{2\pi e} \, \exp{\left(\frac{2}{d} h\left(\bx_0 \mid \bx_t\right)\right)} \, \exp{\left(- I\left(\bx_0 \,;\, E_{\bphi}(\bx_0) \mid \bx_t\right) \right)}.
\end{aligned}
\end{equation}
Finally, taking the logarithm of both sides of inequality and rearranging the terms leads to the result stated in the theorem:
\begin{equation}
\begin{aligned}
&I\left(\bx_0 \,;\, E_{\bphi}(\bx_0) \mid \bx_t\right)\\
\geq &-\log{\Loss_{\bx_0, DSM, \bphi, t}} + \log\left(\frac{d}{2\pi e}\right) + \frac{2}{d} \, h\left(\bx_0 \mid \bx_t\right).
\end{aligned}
\end{equation}
\end{proof}
\begin{figure}[t]
\centering 
\includegraphics[width=.45\textwidth,height=!]{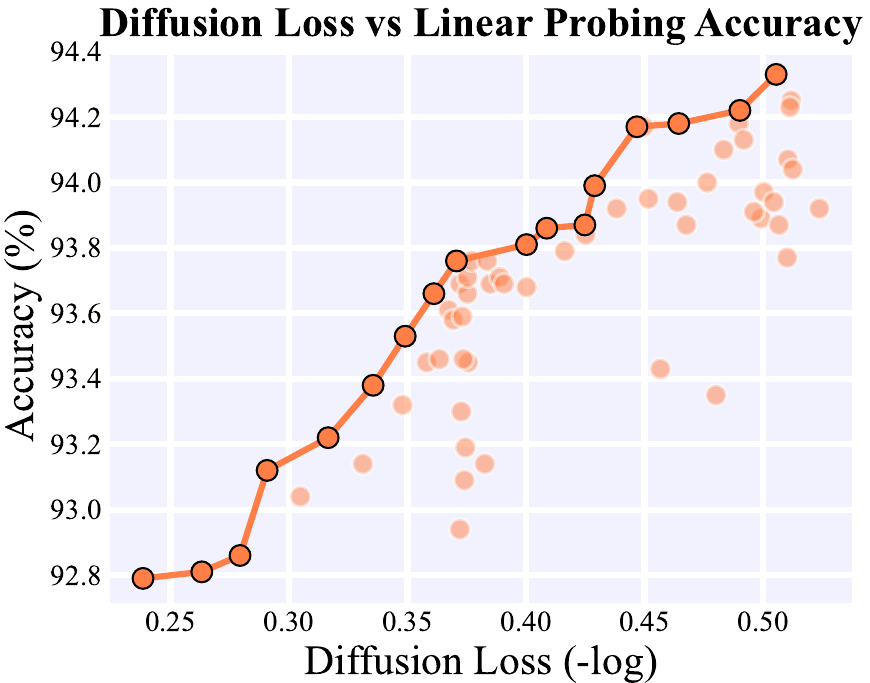}
\caption{The correlation between the negative logarithm of diffusion loss (\textbf{x-axis}) and linear probing accuracy (\textbf{y-axis}) on the Photo dataset.}
\label{fig:loss_acc}
\end{figure}

The proof is in Appendix A. 
Theorem~\ref{thm:diff infomax} indicates that minimizing the diffusion reconstruction objective is equivalent to maximizing a lower bound of conditional mutual information between data and feature. Figure~\ref{fig:loss_acc} illustrates the correlation between diffusion reconstruction loss and linear probing accuracy on downstream tasks. As the diffusion loss decreases, the lower bound of conditional mutual information increases, which in turn corresponds to higher linear probing accuracy. This supports our theory that a lower diffusion loss is associated with more informative representations, leading to improved performance in linear probing on downstream tasks.

\vspace{\topsep}
InfoMax principle\cite{linsker1988self, hjelm2018learning} proposes to choose a representation $f(\bx)$ by maximizing $I(\bx;f(\bx))$. Motivated by Theorem~\ref{thm:diff infomax}, we propose the Diff-InfoMax principle:
\vspace{\topsep}

\noindent\textbf{Diff-InfoMax Principle.} Choosing a representation $f(\bx)$ by maximizing $\int_0^T \lambda(t) I(\bx;f(\bx) | \bx_t) \mathrm{d}t$, where $\bx_t = \alpha_t \bx + \sigma_t \xi$ is a data corrupted by Gaussian Noise and $\lambda(t) \in \Real$ is a weighting function.

The first key distinction between the Diff-InfoMax principle and the original InfoMax principle is that Diff-InfoMax optimizes the conditional mutual information $I(\bx; f(\bx) | \bx_t)$, which quantifies the shared information between $\bx$ and $f(\bx)$, given the knowledge of $\bx_t$. The second difference lies in Diff-InfoMax's use of a multi-level criterion, encouraging the representation to maximize information about $\bx$ while excluding the information from $\bx_t$. By accounting for different noise levels in $\bx_t$, $I(\bx; f(\bx) | \bx_t)$ promotes the representation to capture varying levels of structural detail. Furthermore, we demonstrate that the original InfoMax principle is a special case of the proposed Diff-InfoMax principle.

\begin{remark}
    The original InfoMax principle is a special case of the Diff-InfoMax principle when $\lambda(t) = \delta_T(t)$: $\int_0^T \delta_T(t) I(\bx;f(\bx) | \bx_t) \mathrm{d}t = I(\bx;f(\bx) | \bx_T) = I(\bx;f(\bx))$ because $\bx_T$ is a Gaussian noise independent with $\bx$ and $f(\bx)$.
\end{remark}

Similar to MI, estimating conditional MI is particularly challenging in high-dimensional spaces. We address this by optimizing a tractable lower bound of conditional MI, specifically the DRL objective. We believe the Diff-InfoMax principle opens up new avenues for integrating diffusion models with representation learning. Moreover, there are alternative methods for optimizing variational lower bounds of the conditional MI objective, which we reserve for future exploration.

\subsection{Effects on Frequency Domain}
\paragraph{Frequency-aware Analysis}Several works\cite{yang2023diffusion, si2024freeu, dieleman2024spectral} have noted that during the noising process, the high-frequency components of the data are corrupted first, followed by the low-frequency components. Conversely, in the generation process, low-frequency components are generated initially, with high-frequency components added later. Then the diffusion model performs a role generating high-frequency components given noisy data which mainly consists of low-frequency data. From this frequency domain perspective, $I(\bx; f(\bx) | \bx_t)$ guides the feature extractor to focus on components with frequencies exceeding a certain threshold, with different time $t$ corresponding to different frequency thresholds.
\paragraph{Graph Feature}BWGNN\cite{tang2022rethinking} defines a metric \textit{Energy Ratio} to assess the concentration of graph features in low frequencies. They observe that perturbing graph features with random noise results in a 'right-shift' of energy, indicating a reduced concentration in low frequencies and an increased concentration in high frequencies. This finding aligns with our analysis of the frequency domain. Consequently, DRL operates in the spectral space of graph features, excelling at capturing high-frequency information in these features."

%% file: sections/5_method.tex
\section{The \themodel Approach}

As inspired by the above theoretical insights and to overcome the challenges mentioned in \Cref{sec:intro}, the \themodel framework follows the \emph{explicit-encoder pattern} and couples a graph encoder $E_{\bphi}$ with a conditional diffusion decoder $D_{\theta}$. Given an input graph $\mathcal{G}=(\mathbf{X}, \mathbf{A})$, the encoder achieves perception of both structural and feature information and extracts a compact representation $\mathbf{z}=E_{\bphi}(\mathcal{G})$ for each node. Then, the decoder receives both noisy feature $\bx_t$ and encoded representation $\mathbf{z}$ to reconstruct the original feature $\tilde{\bx}=D_{\theta}(\bx_t, t, \mathbf{z})$. The overall framework is demonstrated in \Cref{fig:framework}. We next introduce the \themodel in detail. 

\subsection{The Graph Encoder}

The encoder module is the core part of our model. Since we are not concerned with generative capabilities, the encoder is the only parameterized module used in downstream tasks, and its capability directly impacts task performance. 
We consider two factors that guide the training lean toward representation learning: one is the \textbf{expressive capacity of the encoder}, which refers to whether it can fully perceive graph data to provide strong representations. 
The other is the \textbf{adequacy of encoder training}, which involves whether the optimization of the objective function can effectively coordinate the optimization of both the encoder and decoder.

For the first factor,  we follow prior work \cite{hou2022graphmae, hou2023graphmae2, zhao2024masked} on the encoder selection, which adopted GAT \cite{velickovic2017graph} and GIN \cite{xu2018powerful} for node and graph tasks, respectively, as both theoretical and empirical evidence demonstrate that they have strong expressive capabilities for graph tasks. This also ensures fair comparison in subsequent experimental analysis. Specifically, their message-passing mechanism can be expressed as:
\begin{equation}
h_v^{(k)} = \mathbf{\mathtt{COMB}}\left( h_v^{(k-1)},\mathbf{\mathtt{AGGR}}\{h_u^{(k-1)}:u \in \mathcal{N}(v)\}\right),
\end{equation}
where $1 \leq k \leq L$ and $h_v^{(k)}$ denotes representation of node $v$ at the $k$-th layer, \(\mathcal{N}(v)\) is the set of neighboring nodes connected to node \(v\) and \(L\) is the number of layers. \( \mathbf{\mathtt{AGGR}}(\cdot)\) and \( \mathbf{\mathtt{COMB}}(\cdot)\) are used for aggregating neighborhood information and combining ego- and neighbor-representations, respectively. For graph-level tasks, the \( \mathbf{\mathtt{READOUT}}(\cdot)\) function aggregates node features from the final iteration to  obtain the entire graph’s representation.

It is worth noting that even given a powerful representation learner, there is a potential risk that the model training may tend to ignore the information in $\mathbf{z}$. This is because the input $\bx$ to the encoder and the reconstruction target by the decoder are the same, which might lead the model to learn a \emph{"shortcut"}. Consider an extreme case where the encoder performs an identity matrix mapping $E_{\bphi}(\cdot) = \mathcal{I(\cdot)}$ on the input features, the optimization objective transforms to $\Loss_{\bx_0, DSM} = \E_t\left\{\lambda(t) \E_{\bx_0} \E_{\bx_t | \bx_0} \left[ \|\bx_{\btheta}(\bx_t, t, \bx_0) - \bx_0 \|^2 \right] \right\}$.
In this scenario, the encoder obtains a poor capability to extract graph semantics, since the loss can easily approach zero. To this end, we randomly zero out partial node features before inputting them into the encoder. 

Formally, let \( \mathbf{X} \in \mathbb{R}^{n \times d} \) be a feature matrix. Define a masking vector \( h_{[mask]} \) consisting of \( n \) Bernoulli random variables with probability $m$, then the modified matrix \( \bf{X'} \) can be expressed as:
\begin{equation}
    h_{[mask]} \sim \text{Bernoulli}(1 - m)^n, \quad
    \mathbf{X'} = \text{diag}(h_{[mask]}) \mathbf{X}.
\end{equation}
Using corrupted node features as input not only effectively prevents the model from learning shortcuts, but also reduces redundancy in attributed graphs. This approach essentially creates a more challenging self-supervision task for learning robust and meaningful representations.

\subsection{The Diffusion Decoder} 
\label{sec:decoder}
\paragraph{Reconstruction objective.} Unlike image features, graph data incorporates feature and structural information, prompting the question of which to prioritize for reconstruction. Previous work in graph SSL has explored both directions: for example, GraphMAE \cite{hou2022graphmae} focuses only on feature information, while another concurrent work, MaskGAE \cite{li2023what}, only targets topological attributes. It is worth noting that in many graph learning datasets, features are often one-hot embeddings, and topology is represented by adjacency matrices—both of which are highly sparse, thus making it difficult to make decisions based on the nature of data. We empirically tested reconstructing features, topology, and their combination. Results in \Cref{tab:ablation} demonstrate that feature reconstruction performs best, outperforming the hybrid approach, with topology-only reconstruction yielding the worst results. Therefore, we choose features \(\bx\) as the target for reconstruction.

\begin{table*}[t]
    \centering
    \caption{Empirical performance of self-supervised representation learning for \underline{node classification} in terms of accuracy (\%, \(\uparrow\)). We highlight the best- and the second-best performing results in \textbf{boldface} and \underline{underlined}, respectively.
    }
    \resizebox{0.9\textwidth}{!}{
    \begin{tabular}{@{}c|c|cccccc@{}}
        \toprule[1.2pt]
            & Dataset &   Cora      & CiteSeer      & PubMed                & Ogbn-arxiv        & Computer               & Photo       \\
         \midrule
        \multirow{2}{*}{Supervised} 
        & GCN     &  81.5$\pm$0.5  & 70.3$\pm$0.7  & 79.0$\pm$0.4                   & 71.7$\pm$0.3    & 86.5$\pm$0.5    & 92.4$\pm$0.2          \\
        & GAT     &  83.0$\pm$0.7  & 72.5$\pm$0.7  & 79.0$\pm$0.3           & 72.1$\pm$0.1     & 86.9$\pm$0.3    & 92.6$\pm$0.4           \\
        \midrule
        \multirow{15}{*}{Self-supervised} 
        & GAE     &  71.5$\pm$0.4  & 65.8$\pm$0.4  & 72.1$\pm$0.5           & 63.6$\pm$0.5               &  85.1 $\pm$ 0.4        & 91.0$\pm$0.2 \\
        & GPT-GNN &  80.1$\pm$1.0  & 68.4$\pm$1.6  & 76.3$\pm$0.8 & - & - & -\\
        & GATE    &  83.2$\pm$0.6  & 71.8$\pm$0.8  & 80.9$\pm$0.3           & -                 & -             & -   \\ 
        & DGI     &  82.3$\pm$0.6  & 71.8$\pm$0.7  & 76.8$\pm$0.6           & 70.3$\pm$0.2 & 84.0$\pm$0.5     & 91.6$\pm$0.2 \\
        & MVGRL   & 83.5$\pm$0.4   & 73.3$\pm$0.5  & 80.1$\pm$0.7           & -               & 87.5$\pm$0.1 & 91.7$\pm$0.1 \\
        & GRACE   & 81.9$\pm$0.4   & 71.2$\pm$0.5  & 80.6$\pm$0.4           & 71.5$\pm$0.1  & 86.3$\pm$0.3  &    92.2$\pm$0.2\\  
        & BGRL    & 82.7$\pm$0.6   & 71.1$\pm$0.8  & 79.6$\pm$0.5           & 71.6$\pm$0.1  & 89.7$\pm$0.3  & 92.9$\pm$0.3         \\
        & InfoGCL  & 83.5$\pm$0.3   & \underline{73.5} $\pm$0.4  & 79.1$\pm$0.2  & - & - & - \\
        & CCA-SSG & 84.0$\pm$0.4   & 73.1$\pm$0.3  & 81.0$\pm$0.4  & 71.2$\pm$0.2  & 88.7$\pm$0.3  & 93.1$\pm$0.1   \\
        & GraphMAE& 84.2$\pm$0.4 & 73.4$\pm$0.4 &  81.1$\pm$0.4 & 71.8$\pm$0.2 & 88.6$\pm$0.2 & 93.6  $\pm$ 0.2 \\
        & GraphMAE2& 84.1$\pm$0.6 & 73.1$\pm$0.4 & 80.9$\pm$0.5 & \underline{71.8$\pm$0.0} & 89.2$\pm$0.4 & 93.3 $\pm$ 0.2 \\
        & AUG-MAE& 84.3$\pm$0.4 & 73.2$\pm$0.4 & 81.4$\pm$0.4 & \underline{71.9$\pm$0.2} & 89.4$\pm$0.2 & 93.1 $\pm$ 0.3 \\
        & MaskGAE$_{edge}$& 83.8$\pm$0.3 & 72.9$\pm$0.2 & 82.7$\pm$0.3 & 71.0$\pm$0.3 & 89.4$\pm$0.1 & 93.3 $\pm$ 0.0 \\
        & MaskGAE$_{path}$& 84.3$\pm$0.3 & 73.8$\pm$0.8 & \underline{83.6$\pm$0.5} & 71.2$\pm$0.3 & 89.5$\pm$0.1 & 93.3 $\pm$ 0.1 \\
        & DDM     & 83.4$\pm$0.2 & 72.5$\pm$0.3 & 79.6$\pm$0.8 & 71.3$\pm$0.2 & \underline{89.9$\pm$0.2} & \underline{93.8$\pm$0.2} \\
        & Bandana & \underline{84.5$\pm$0.3} & 73.6$\pm$0.2 & \bf 83.7$\pm$0.5 & 71.1$\pm$0.2 & 89.6$\pm$0.1 & 93.4 $\pm$ 0.1 \\
        \cmidrule{2-8}
         & \mk\themodel   & \mk\bf 84.8$\pm$0.4  & \mk\bf{74.3$\pm$0.4}  & \mk81.0$\pm$0.6  & \mk\bf 72.1$\pm$0.2 & \mk\bf 91.3$\pm$0.2    & \mk\bf94.2$\pm$0.1\\
        \bottomrule[1.2pt]
        \end{tabular}
        }
        \label{tab:node_clf}
\end{table*}

\paragraph{Customized instantiation of decoder.}
In decoder design, we draw on the experience of using the U-Net architecture from the visual domain as a backbone model for diffusion training. The U-Net architecture \cite{ronneberger2015unet} provides representations of different granularities through up- and down-sampling \cite{si2024freeu}. Additionally, it aligns well with the strict dimensional requirements of diffusion models. Specifically, when handling graph-level tasks, we propose Graph-UNet, which adopts GNN layers to replace the convolutional layers in the vanilla U-Net. In this context, each graph in a mini-batch can be likened to an image in a visual diffusion model; by uniformly sampling time step \(t \sim \text{Uniform}(0,T)\) within a mini-batch, we ensure that the level of feature noise within each graph remains consistent.

However, for node-level tasks, if we instantiate the decoder with GNNs, it becomes problematic to use different time steps for different nodes, as this would lead to message passing propagating node information at varying noise levels.  Therefore, to enable the model to clearly perceive distinct noise levels and conduct training in a principled manner, we replace the GNN layers with the MLP network. 

\paragraph{Architecture of Graph-Unet}
As illustrated on the right side of \Cref{fig:framework}, our decoder adopts a UNet-like architecture, comprising a contracting path (left side) and an expansive path (right side). However, since up-sampling and down-sampling operations cannot be directly applied to graph data, we instead represent the granularity of modeling through dimensional reduction and expansion. 
Specifically, due to the requirement of the diffusion model that the input and output dimensions match the original feature dimensions, we introduce additional input and output layers to perform dimensional mappings. In the contracting path, repeated dimensional reduction is performed using either GNN layers or MLP layers, depending on different task types, which halves the number of hidden dimensions at each step. In the expansive path, dimensional expansion is repeated, but before each mapping, the hidden state of the corresponding contracting path with the same dimension is added via skip connections, which differs from the original UNet’s concatenation.

It is also important to note that, in addition to the noisy data \(\bx_t\), the decoder also receives the condition $\mathbf{z}$ and time $t$ as inputs. We encode the time information using two linear layers with SiLU activation \cite{elfwing2018sigmoid}, and employ positional encoding to enable the model to distinguish temporal order. Furthermore, a key challenge is how to fuse \(\bx_t\), \(\mathbf{z}\), and \(t\). Based on experimental results, the optimal approach for node-level tasks is to directly sum these three components after encoding, as shown below:
\begin{equation}
    \mathbf{h}^{(l+1)} = \mathbf{h}^{(l)} + \mathbf{\mathtt{MLP}}_t(t) + \mathbf{\mathtt{MLP}}_z(\mathbf{z})
\end{equation}
where $\mathbf{\mathtt{MLP}}_t(\cdot)$ and $\mathbf{\mathtt{MLP}}_z(\cdot)$ are both MLP layer to achieve dimensional mapping.

For graph-level tasks, we follow the approach commonly used in the field of computer vision, utilizing Adaptive Normalization layers \cite{dhariwal2021diffusion, hudson2024soda} to fuse the three components:
\begin{align*}
    \mathbf{h}^{(l+1)} &= \mathtt{AdaNorm}(\mathbf{h}^{(l)}, \mathbf{z}, t) \\
    &=\mathbf{z}_s(t_s\mathtt{LayerNorm}(\mathbf{h}^{(l)})+t_b)+\mathbf{z}_b
\end{align*}
where $(t_s, t_b)$ and $(\mathbf{z}_s, \mathbf{z}_b)$ are obtained by linear projection.

%% file: sections/6_exp.tex
\section{Experiments}
\begin{table*}[t]
    \setlength{\tabcolsep}{10pt}
    \centering
    \caption{Experiment results in self-supervised representation learning for \underline{graph classification}. \textmd{We report accuracy (\%) for all datasets.} We highlight the best- and the second-best performing results in \textbf{boldface} and \underline{underlined}, respectively.
    }
    \resizebox{0.9\textwidth}{!}{ 
    \begin{tabular}{c|c|ccccc}
        \toprule[1.2pt]
              & Dataset  & IMDB-B     & IMDB-M     & PROTEINS   & COLLAB     & MUTAG       \\ 

        \midrule
        \multirow{2}{*}{Supervised}
        & GIN         & 75.1$\pm$5.1   & 52.3$\pm$2.8   & 76.2$\pm$2.8   & 80.2$\pm$1.9   & 89.4$\pm$5.6     \\ 
        & DiffPool    & 72.6$\pm$3.9 &  -           &  75.1$\pm$3.5   & 78.9$\pm$2.3 & 85.0$\pm$10.3 \\
        \midrule
        \multirow{2}{*}{Graph Kernels}
        & WL          & 72.30$\pm$3.44 & 46.95$\pm$0.46 & 72.92$\pm$0.56 & - & 80.72$\pm$3.00 \\ 
        & DGK         & 66.96$\pm$0.56 & 44.55$\pm$0.52 & 73.30$\pm$0.82 & - & 87.44$\pm$2.72 \\ 
        \midrule
        \multirow{9}{*}{Self-supervised}
        & graph2vec   & 71.10$\pm$0.54 & 50.44$\pm$0.87 & 73.30$\pm$2.05 & -              & 83.15$\pm$9.25 \\   
        
        & Infograph   & 73.03$\pm$0.87 & 49.69$\pm$0.53 & 74.44$\pm$0.31 & 70.65$\pm$1.13 & 89.01$\pm$1.13 \\
        
        & GraphCL     & 71.14$\pm$0.44 & 48.58$\pm$0.67 & 74.39$\pm$0.45 & 71.36$\pm$1.15 & 86.80$\pm$1.34 \\
        
        & JOAO        & 70.21$\pm$3.08 & 49.20$\pm$0.77     & \underline{74.55$\pm$0.41} & 69.50$\pm$0.36 & 87.35$\pm$1.02 \\
        & GCC         & 72.0           & 49.4           & -    & 78.9    &  - \\
        & MVGRL       & 74.20$\pm$0.70   & 51.20$\pm$0.50   & -              & -              & {89.70$\pm$1.10}  \\
        & InfoGCL     & {75.10$\pm$0.90}   & {51.40$\pm$0.80}   &  -   & {80.00$\pm$1.30}   &  \underline{91.20$\pm$1.30}  \\
        & GraphMAE      &  75.52$\pm$0.66 &  51.63$\pm$0.52 &  \underline{75.30$\pm$0.39} &  80.32$\pm$0.46 & 88.19$\pm$1.26 \\
        & AUG-MAE      &  \underline{75.56$\pm$0.61} &  51.80$\pm$0.86 &  \bf{75.83$\pm$0.24}&  80.48$\pm$0.50 & 88.28$\pm$0.98 \\
        & DDM        & 74.05$\pm$0.17    &  \underline{52.02$\pm$0.29}    & 71.61$\pm$0.56     & \underline{80.70$\pm$0.18}   &90.15$\pm$0.46  \\
        \cmidrule{2-7}
        & \mk\themodel        & \mk\bf{76.20$\pm$0.23}    & \mk\bf{52.4$\pm$0.37}    & \mk74.36$\pm$0.12     & \mk\bf{81.28$\pm$0.15}        &\mk\bf{91.46$\pm$0.26}  \\
        \bottomrule[1.2pt]
    \end{tabular}
    }
    \label{tab:graph_clf}
\end{table*}
\subsection{Experimental Setup}

\textbf{Datasets.} 
Our experiments primarily involve node-level and graph-level datasets. 
For node classification tasks, we select 6 datasets drawn from various domains for evaluation. These include three citation networks: \texttt{Cora}, \texttt{CiteSeer}, and \texttt{PubMed} \cite{sen2008collective}; two co-purchase graphs: \texttt{Photo} and \texttt{Computer} \cite{shchur2018pitfalls}; and a large dataset from the Open Graph Benchmark: \texttt{arXiv} \cite{hu2020open}. The evaluation datasets represent real-world networks and graphs from diverse fields.
For graph classification tasks, we select 5 datasets for training and testing: \texttt{IMDB-B}, \texttt{IMDB-M}, \texttt{PROTEINS}, \texttt{COLLAB}, and \texttt{MUTAG} \cite{yanardag2015deep}. Each dataset comprises a collection of graphs, with each graph assigned a label. In graph classification tasks, the node degrees are used as attributes for all datasets. These features are processed using one-hot encoding as input to the model. 

\textbf{Evaluation protocols.}
We follow the experimental settings from \cite{hassani2020contrastive,velickovic2019deep}. First, we train a GNN encoder and a decoder using the proposed \themodel in an unsupervised manner. Then, we freeze the encoder parameters to infer the node representations. We train a linear classifier to evaluate the representation quality and report the average accuracy on test nodes over 20 random initializations. For node classification tasks, we use the public data splits of \texttt{Cora}, \texttt{Citeseer}, and \texttt{PubMed} as specified in \cite{hassani2020contrastive, thakoor2021large, velickovic2019deep} and adopt GAT \cite{velickovic2017graph} as the graph encoder. For graph classification tasks, we follow the experimental setup by \cite{hou2022graphmae} and adopt the GIN \cite{xu2018powerful} as the graph encoder. We feed the graph-level representations into the downstream LIBSVM classifier \cite{chang2001libsvm} to predict labels. The average 10-fold cross-validation accuracy and standard deviation after 5 runs.

\textbf{Implementation details.} 
In our study, we employ either Adam \cite{kingma2014adam} or AdamW \cite{loshchilov2017decoupled} as the optimizer, complemented by a cosine annealing scheduler \cite{loshchilov2016sgdr} to enhance model convergence across different datasets. Moreover, we configure the learning rate for the encoder to be twice that of the decoder, a strategy that has demonstrated empirical effectiveness in promoting training stability. In terms of the noise schedule, we explore several candidate approaches, including sigmoid, linear, and inverted schedules, ultimately selecting the most appropriate method based on their performance for each dataset.

\subsection{Node Classification}

For comprehensive comparison, we select the following three groups of SSL methods as primary baselines in our experiments.
\ding{172} Auto-encoding methods: GAE \cite{kipf2016variational}, GATE \cite{salehi2019graph}, GraphMAE\cite{hou2022graphmae}, GraphMAE2\cite{hou2023graphmae2}, MaskGAE\cite{li2023what},
AUG-MAE\cite{wang2024rethinking},
Bandana\cite{zhao2024masked}
\ding{173} Contrastive methods: GRACE \cite{zhu2021graph}, CCA-SSG \cite{zhang2021canonical}, InfoGCL \cite{xu2021infogcl}, DGI\cite{velickovic2019deep}, MVGRL \cite{hassani2020contrastive}, BGRL \cite{thakoor2021large}, GCC \cite{qiu2020gcc}
\ding{174} Others: GPT-GNN \cite{hu2020gpt}, DDM \cite{yang2024directional}. 
Detailed hyper-parameter configurations are provided in Appendix B.
The performance of 6 linear probing node classification tasks is summarized in \Cref{tab:node_clf}. The results not reported are due to unavailable code or out-of-memory.
Generally, it can be found from the table that our \themodel shows strong empirical performance across all datasets, delivering five out of six state-of-the-art results.
The outstanding results validate the superiority of our proposed model.

We make other observations as follows: \textbf{\emph{(i)}} Note that previous work has already achieved pretty high performance. For example, the current state-of-the-art DDM only obtains a 0.24\% absolute improvement over the second-best baseline, Bandana,  in terms of average accuracy on the \texttt{Computer} dataset. Our work pushes that boundary with absolute improvement up to 1.46\% over DDM. \textbf{\emph{(ii)}} Our method surpasses the supervised training baseline on almost all tasks. For instance, in the \texttt{Computer} dataset, the GAT baseline achieves an accuracy of 86.9 under fully supervised training; however, \themodel improves upon this by 4.4 percentage points. Interestingly, this further corroborates our theoretical findings presented in \Cref{sec:extrac_info} and illustrated in \Cref{fig:loss_comp}. The consistency between our empirical results and theoretical analysis reinforces the robustness of our model. It demonstrates that our proposed model can obtain meaningful and high-quality embeddings.

\subsection{Graph Classification}

For graph classification tasks, we further include the graph kernel methods \cite{shervashidze2011weisfeiler, yanardag2015deep} and graph2vec \cite{narayanan2017graph2vec} following \cite{hou2022graphmae}. 
Detailed hyper-parameter configurations are provided in Appendix B. 
The performance of \themodel on 5 datasets is summarized in \Cref{tab:graph_clf}.  It can be observed that our method demonstrates performant results on different tasks, achieving state-of-the-art results on 4 out of 5 datasets. This further indicates that \themodel, as a new class of generative SSL, holds significant potential in representation learning. Furthermore, similar to observations in node classification, our method also outperforms fully supervised counterparts.

\subsection{Ablation Study}

\paragraph{Effect of different components}
To demonstrate the necessity of each module in our model, we conduct ablation study to validate the different components of \themodel. Specifically, we consider three aspects for ablation: reconstruction objectives, masking strategies, and decoder selection. We select \texttt{Cora}, \texttt{Computer}, and \texttt{Photo} for node-level tasks, and \texttt{IMDB-B}, \texttt{COLLAB}, and \texttt{MUTAG} for graph-level tasks. 
The experimental results are presented in \Cref{tab:ablation}. 
\begin{table}[h]
    \centering
    \caption{Ablation of different components.}
    \resizebox{0.45\textwidth}{!}{ 
    \begin{tabular}{c|ccc}
        \toprule[1.2pt]
        Node-level  & Cora     & Computer & Photo\\ 
        \midrule
        $\bf A$ Recons.       & 77.6 & 86.2 & 91.7\\
        $\bf A+X$ Recons.      & 80.1 & 87.4 & 92.2\\
        w/o Mask               & 82.5 & 88.5 & 92.5\\
        w. GAT decoder         & 83.2 & 89.8 & 92.9\\
        \mk\themodel         & \mk\bf84.8 & \mk\bf91.3 & \mk\bf94.2\\
        \bottomrule
        \toprule
        Graph-level  & IMDB-B     & COLLAB & MUTAG\\ 
        \midrule
        $\bf A$ Recons.       &70.2	&71.5	&83.6\\
        $\bf A+X$ Recons.     &71.6	&77.6	&86.8\\
        w/o Mask              &75.8	&81.2	&91.5\\
        w. MLP decoder        &74.5	&79.9	&88.5\\
        \mk\themodel             &\mk\bf 76.2	&\mk\bf 81.3 &\mk\bf 91.5\\
        \bottomrule[1.2pt]
    \end{tabular}
    }
    \label{tab:ablation}
\end{table}

Our observations are as follows: 
\textbf{\emph{(i)}} The performance of reconstructing only feature (i.e., the Graffe model) surpasses that of the mixed reconstruction, with the worst performance occurring when reconstructing only topology. This suggests that explicitly reconstructing structural information leads to performance degradation. 
\textbf{\emph{(ii)}} The masking strategy is particularly critical for node-level tasks, as its removal results in significant performance drops, while the impact is less noticeable for graph-level tasks. 
\textbf{\emph{(iii)}} The choice of decoder layers is critical for different task types. For node-level tasks, using an MLP layer yields better results compared to a GAT layer, while the opposite is true for graph-level tasks. This aligns with our intuitive analysis in \Cref{sec:decoder}, indicating that the propagation of noise is detrimental to diffusion representation learning.

\paragraph{Effect of mask ratio}

Since mask strategy is a crucial component of our framework, it is necessary to evaluate how to choose a proper $m$. We conduct an empirical analysis on \texttt{Cora}, \texttt{Computer} and \texttt{MUTAG} dataset and consider a candidate list covering the value ranges of $m$: [0, 0.1, 0.3, 0.5, 0.7, 0.9].
As shown in \Cref{fig:mask}, the optimal masking choice varies across different datasets. 
For the \texttt{Cora} and \texttt{Computer} datasets, the best performance is achieved when \(m = 0.7\), whereas on the \texttt{MUTAG} dataset, the best results are obtained without applying any masking. 
Moreover, a higher mask ratio even leads to performance decline on graph-level tasks. This suggests that the selection of the mask ratio should be tuned according to the specific tasks, as there is no one-size-fits-all solution.
\begin{figure}[t]
\centering
\includegraphics[width=0.4\textwidth]{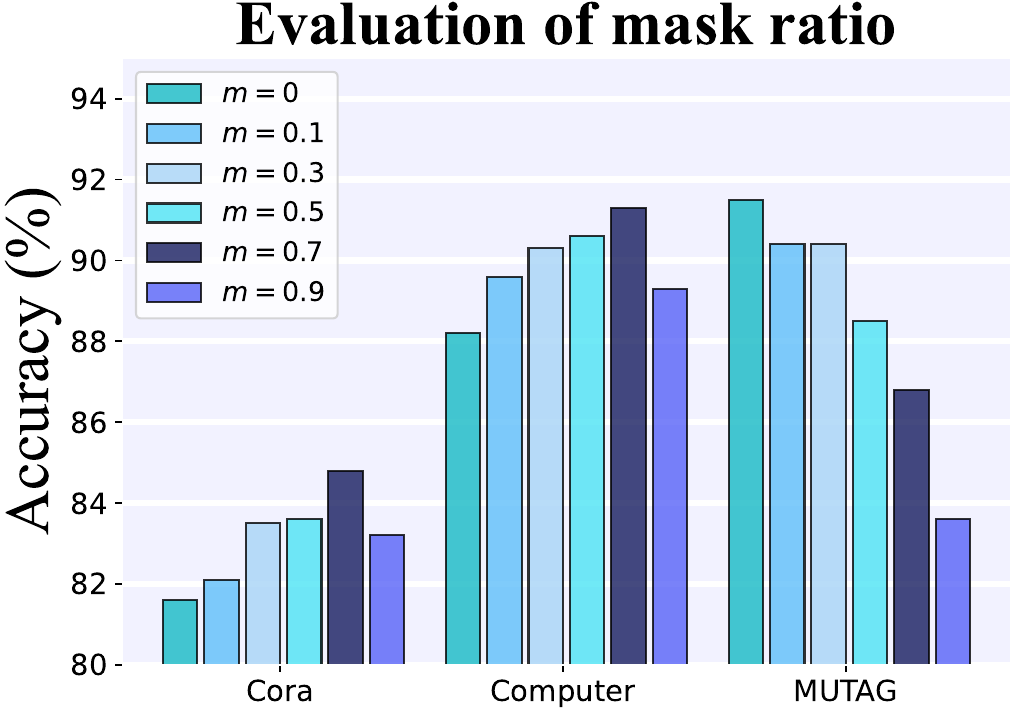}
\caption{The effect of mask ratio $m$ on \texttt{Cora}, \texttt{Computer} and \texttt{MUTAG} dataset.}
\label{fig:mask}
\vspace{-10pt}
\end{figure}

\paragraph{Ablation study on encoder backbone}

To evaluate how much impact the choice of encoder has on the performance of \themodel and other baselines, we conduct ablation studies on the encoder backbone using three classic datasets: \texttt{Cora}, \texttt{Citeseer}, and \texttt{Computer}. We chose GRACE \cite{zhu2021graph} and CCA-SSG  \cite{zhang2021canonical} as baselines for contrastive learning and GraphMAE \cite{hou2022graphmae}, MaskGAE \cite{li2023what}, and Bandana \cite{zhao2024masked} as baselines for the MAE family. The results are shown in \Cref{tab:abl_encoder}.

\begin{table}[h!]
    \centering
    \renewcommand{\arraystretch}{1.2}
    \caption{Ablation study on different encoder design.}
    \resizebox{0.45\textwidth}{!}{ 
    \begin{tabular}{@{}lcccccc@{}}
        \toprule[1.2pt]
        \multirow{2}{*}{Method} & \multicolumn{2}{c}{Cora} & \multicolumn{2}{c}{Citeseer} & \multicolumn{2}{c}{Computer} \\ 
        \cmidrule(r){2-3} \cmidrule(r){4-5} \cmidrule(r){6-7}
        & GCN         & GAT         & GCN         & GAT         & GCN         & GAT  \\ 
        \midrule[1pt]
        GRACE                            & 81.9 & 81.0 & 71.2 & 71.5 & 86.3 & 86.2 \\
        GraphMAE                         & 82.5 & 84.2 & 72.6 & 73.4 & 86.5 & 88.6 \\
        CCA-SSG                          & 84.0 & 82.7 & 73.1 & 72.3 & 88.7 & 85.5 \\
        MaskGAE$_{edge}$                & 83.8 & 82.0 & 72.9 & 72.0 & 89.4 & 87.7 \\
        Bandana                          & \bf84.5 & 83.1 & \bf73.6 & 73.7 & 89.6 & 89.2 \\
        \mk\themodel                    & \mk83.2 & \mk\bf84.8 & \mk73.2 & \mk\bf74.3 & \mk\bf90.8 & \mk\bf91.3 \\
        \bottomrule[1.2pt]
        \label{tab:abl_encoder}
    \end{tabular}
    }
\end{table}

The results show significant performance declines for many methods when substituting GCN for GAT, such as CCA-SSG, MaskGAE, and Bandana on Cora and Citeseer datasets, which also aligns with observations in MaskGAE \cite{li2023what}. In contrast, for GraphMAE and \themodel, switching their GAT backbones to GCN also causes a drop in performance. We believe different SSL methods have distinct encoder preferences and using GAT or GCN as the encoder in graph SSL is not universally optimal.

\paragraph{Ablation study on Graph-Unet backbone}
As mentioned in \Cref{sec:decoder}, we chose the Unet structure because it can capture information at different granularities while strictly ensuring input-output dimensional consistency. During our early exploration, we also tested using a simple MLP or GNN as the decoder. The experimental results on Cora, Photo, and IMDB-B datasets are shown in \Cref{tab:abl_decoder}. It is worth noting that the GNN decoder adopts the same architecture as the encoder: GAT for node-level tasks and GIN for graph-level tasks.

\begin{table}[h]
    \centering
    \caption{Ablation study on different decoder design.}
     \setlength{\tabcolsep}{7pt}
     \renewcommand{\arraystretch}{1.2}
     \resizebox{0.45\textwidth}{!}{
    \begin{tabular}{lccc}
        \toprule[1.2pt]
        Decoder      & Cora       & Computer   & IMDB-B      \\ 
        \midrule
        MLP                    & 82.6$\pm$0.5        & 89.1$\pm$0.1        & 75.0$\pm$0.6         \\ 
        GNN (GAT/GIN)          & 80.2$\pm$0.3        & 88.1$\pm$0.1        & 74.5$\pm$0.5         \\ 
        \mk{Graph-Unet}             & \mk\bf84.8$\pm$0.4        & \mk\bf91.3$\pm$0.2        & \mk\bf76.2$\pm$0.2         \\ 
        \bottomrule[1.2pt]
        \label{tab:abl_decoder}
    \end{tabular}
    }
\end{table}

We can observe that using either an MLP or GNN as the decoder results in significantly poorer performance compared to the Graph-Unet. Moreover, for node-level tasks, employing a GNN as the decoder leads to a substantial performance drop. This observation aligns with our analysis in \Cref{sec:decoder}, where GNNs can cause interference among nodes due to varying degrees of noise introduced during the diffusion process.

%% file: sections/7_conclusion.tex
\section{Conclusion}

In this paper, we introduce \themodel, a self-supervised diffusion representation learning (DRL) framework designed for graphs, achieving state-of-the-art performance on self-supervised graph representation learning tasks. We establish the theoretical foundations of DRL and prove that the denoising objective is a lower bound for the conditional mutual information between data and its representations. We propose the Diff-InfoMax principle, an extension of the standard InfoMax principle, and demonstrate that DRL implicitly follows it. Based on these theoretical insights and customized design for graph data, \themodel excels in node and graph classification tasks.

%% file: sections/appendix.tex
{\appendices
\onecolumn
\section{Proofs}\label{app:proofs}
\subsection{Proof of Theorem 1}
\setcounter{theorem}{0}
\begin{theorem}
    The denoising score matching objective $\Loss_{\bx_0, DSM}$ has a \textbf{strictly positive} lower bound, regardless of the network capacity and expressive power
    \begin{equation}
    \begin{aligned}
\min_{\bx_{\theta}} \Loss_{\bx_0, DSM} = &\min_{\bx_{\theta}}\E_t\left\{\lambda(t) \E_{\bx_0} \E_{\bx_t | \bx_0} \left[ \|\bx_{\btheta}(\bx_t, t) - \bx_0 \|^2 \right] \right\}\\
= &  \E_t\left\{\lambda(t)\E_{\bx_t}\left[\Tr(\Cov[\bx_0 | \bx_t])\right]\right\} > 0.
    \end{aligned}
    \end{equation}
    The conditioned denoising score matching objective objective $\Loss_{\bx_0, DSM, \bphi}$ has a \textbf{non-negative} lower bound, i.e.
    \begin{equation}
        \min_{\bx_{\theta}} \Loss_{\bx_0, DSM, \bphi} = \E_t\left\{\lambda(t)\E_{\bx_0,\bx_t}\left[\Tr(\Cov[\bx_0 | \bx_t, E_{\bphi}(\bx_0)])\right]\right\} \geq 0.
    \end{equation}
\end{theorem}
\begin{proof}
    \begin{equation}
        \begin{aligned}
            &\argmin\limits_{\bx_{\theta}} \Loss_{\bx_0, DSM}\\
            =& \argmin\limits_{\bx_{\theta}} \E_t\left\{\lambda(t) \E_{\bx_0} \E_{\bx_t | \bx_0} \left[ \|\bx_{\btheta}(\bx_t, t) - \bx_0 \|^2 \right] \right\} \\
            =& \argmin\limits_{\bx_{\theta}} \E_t\left\{\lambda(t)\E_{\bx_0,\bx_t}\left[ \|
            \bx_{\btheta}(\bx_t, t) - \E [\bx_0 | \bx_t] + \E [\bx_0 | \bx_t] - \bx_0 \|^2 \right]\right\} \\
            =& \argmin\limits_{\bx_{\theta}} \E_t\bigl\{ \lambda(t)\E_{\bx_0,\bx_t}\left[\|
            \bx_{\btheta}(\bx_t, t) - \E [\bx_0 | \bx_t]\|^2 + 2\langle \bx_{\btheta}(\bx_t, t) - \E [\bx_0 | \bx_t],\E [\bx_0 | \bx_t] - \bx_0\rangle \right]\\
            & +\lambda(t)\E_{\bx_0,\bx_t}\left[\|\E [\bx_0 | \bx_t] - \bx_0 \|^2\right] \bigl\}\\
            =& \argmin\limits_{\bx_{\theta}}\E_t\bigl\{ \lambda(t)\E_{\bx_0,\bx_t}\left[\|
            \bx_{\btheta}(\bx_t, t) - \E [\bx_0 | \bx_t]\|^2 + 2\langle\bx_{\btheta}(\bx_t, t) - \E [\bx_0 | \bx_t],\E [\bx_0 | \bx_t] - \bx_0\rangle \right]\bigl\}. \\
        \end{aligned}
    \end{equation}
    Note that
    \begin{equation}
        \begin{aligned}
            & \E_{\bx_0,\bx_t}\left[ \langle \bx_{\btheta}(\bx_t, t) - \E [\bx_0 | \bx_t],\E [\bx_0 | \bx_t] - \bx_0\rangle\right] \\
            = & \E_{\bx_t}\E_{\bx_0|\bx_t} \left[ \langle \bx_{\btheta}(\bx_t, t) - \E [\bx_0 | \bx_t],\E [\bx_0 | \bx_t] - \bx_0\rangle\right] \\
            = & \E_{\bx_t}\left[\langle\bx_{\btheta}(\bx_t, t) - \E [\bx_0 | \bx_t],\E_{\bx_0|\bx_t} \left[ \E [\bx_0 | \bx_t] - \bx_0\right] \rangle \right].\\
        \end{aligned}
    \end{equation}
    Due to the property of conditional expectation, we have that
    \begin{equation}
        \begin{aligned}
            \E_{\bx_0|\bx_t} \left[ \E [\bx_0 | \bx_t] - \bx_0\right]=\E [\bx_0 | \bx_t] - \E [\bx_0 | \bx_t]=0.
        \end{aligned}
    \end{equation}
    Thus we have
    \begin{equation}
        \E_{\bx_0,\bx_t}\left[ \langle \bx_{\btheta}(\bx_t, t) - \E [\bx_0 | \bx_t],\E [\bx_0 | \bx_t] - \bx_0\rangle\right] = 0.
    \end{equation}
    Thus
    \begin{equation}
        \begin{aligned}
            &\argmin\limits_{\bx_{\theta}} \Loss_{\bx_0, DSM}\\
            =& \argmin\limits_{\bx_{\theta}}\E_t\bigl\{ \lambda(t)\E_{\bx_0,\bx_t}\left[\|
            \bx_{\btheta}(\bx_t, t) - \E [\bx_0 | \bx_t]\|^2 + 2\langle\bx_{\btheta}(\bx_t, t) - \E [\bx_0 | \bx_t],\E [\bx_0 | \bx_t] - \bx_0\rangle \right]\bigl\} \\
            =& \argmin\limits_{\bx_{\theta}}\E_t\bigl\{ \lambda(t)\E_{\bx_0,\bx_t}\left[\|
            \bx_{\btheta}(\bx_t, t) - \E [\bx_0 | \bx_t]\|^2 \right]\bigl\}\\
            =& \E [\bx_0 | \bx_t].
        \end{aligned}
    \end{equation}
    Substitute the minimizer of $\Loss_{\bx_0, DSM}$ into it, we get the minimum of $\Loss_{\bx_0, DSM}$
    \begin{equation}
        \begin{aligned}
            &\min_{\bx_{\theta}} \Loss_{\bx_0, DSM}\\
            =& \min_{\bx_{\theta}} \E_t\left\{\lambda(t) \E_{\bx_0} \E_{\bx_t | \bx_0} \left[ \|
            \bx_{\btheta}(\bx_t, t) - \bx_0 \|^2 \right] \right\} \\ 
            =& \E_t\left\{\lambda(t) \E_{\bx_0} \E_{\bx_t | \bx_0} \left[ \|
            \E [\bx_0 | \bx_t] - \bx_0 \|^2 \right] \right\}\\
            =& \E_t\left\{\lambda(t)\E_{\bx_t}\E_{\bx_0|\bx_t} \left[ 
            (\E [\bx_0 | \bx_t] - \bx_0)^T(\E [\bx_0 | \bx_t] - \bx_0)  \right]\right\}\\
            =& \E_t\left\{\lambda(t)\E_{\bx_t}\E_{\bx_0|\bx_t} \left[ 
            \Tr((\E [\bx_0 | \bx_t] - \bx_0)^T(\E [\bx_0 | \bx_t] - \bx_0))  \right]\right\}\\
            =& \E_t\left\{\lambda(t)\E_{\bx_t}\E_{\bx_0|\bx_t} \left[ 
            \Tr((\E [\bx_0 | \bx_t] - \bx_0)(\E [\bx_0 | \bx_t] - \bx_0)^T)  \right]\right\}\\
            =& \E_t\left\{\lambda(t)\E_{\bx_t}\left[\Tr(\E_{\bx_0|\bx_t} \left[ 
            (\E [\bx_0 | \bx_t] - \bx_0)(\E [\bx_0 | \bx_t] - \bx_0)^T  \right])\right]\right\}\\
            =& \E_t\left\{\lambda(t)\E_{\bx_t}\left[\Tr(\Cov[\bx_0 | \bx_t])\right]\right\} > 0.
        \end{aligned}
    \end{equation}
The minimum is strictly positive for non-degenerated distributions $\bx_0|\bx_t$.

The proof of conditioned denoising score matching objective is similar.

    \begin{equation}
        \begin{aligned}
            &\argmin\limits_{\bx_{\theta}} \Loss_{\bx_0, DSM, \bphi}\\
            =& \argmin\limits_{\bx_{\theta}} \E_t\left\{\lambda(t) \E_{\bx_0} \E_{\bx_t | \bx_0} \left[ \|\bx_{\btheta}(\bx_t, t, E_{\bphi}(\bx_0)) - \bx_0 \|^2 \right] \right\} \\
            =& \argmin\limits_{\bx_{\theta}} \E_t\left\{\lambda(t)\E_{\bx_0,\bx_t}\left[ \|
            \bx_{\btheta}(\bx_t, t, E_{\bphi}(\bx_0)) - \E [\bx_0 | \bx_t, E_{\bphi}(\bx_0)] + \E [\bx_0 | \bx_t, E_{\bphi}(\bx_0)] - \bx_0 \|^2 \right]\right\} \\
            =& \argmin\limits_{\bx_{\theta}} \E_t\bigl\{ \lambda(t)\E_{\bx_0,\bx_t}\left[\|
            \bx_{\btheta}(\bx_t, t, E_{\bphi}(\bx_0)) - \E [\bx_0 | \bx_t, E_{\bphi}(\bx_0)]\|^2\right] +  \\
            & +2\lambda(t)\E_{\bx_0,\bx_t}\left[\langle\bx_{\btheta}(\bx_t, t, E_{\bphi}(\bx_0)) - \E [\bx_0 | \bx_t, E_{\bphi}(\bx_0)],\E [\bx_0 | \bx_t, E_{\bphi}(\bx_0)] - \bx_0\rangle\right]\\
            & +\lambda(t)\E_{\bx_0,\bx_t}\left[\|\E [\bx_0 | \bx_t, E_{\bphi}(\bx_0)] - \bx_0 \|^2\right] \bigl\}\\
            =& \argmin\limits_{\bx_{\theta}}\E_t\bigl\{ \lambda(t)\E_{\bx_0,\bx_t}\left[\|
            \bx_{\btheta}(\bx_t, t, E_{\bphi}(\bx_0)) - \E [\bx_0 | \bx_t, E_{\bphi}(\bx_0)]\|^2\right] \\
            & + 2\lambda(t)\E_{\bx_0,\bx_t}\left[\langle\bx_{\btheta}(\bx_t, t, E_{\bphi}(\bx_0)) - \E [\bx_0 | \bx_t, E_{\bphi}(\bx_0)],\E [\bx_0 | \bx_t, E_{\bphi}(\bx_0)] - \bx_0\rangle \right]\bigl\}. \\
        \end{aligned}
    \end{equation}
    Note that
    \begin{equation}
        \begin{aligned}
            & \E_{\bx_0,\bx_t}\left[ \langle\bx_{\btheta}(\bx_t, t, E_{\bphi}(\bx_0)) - \E [\bx_0 | \bx_t, E_{\bphi}(\bx_0)],\E [\bx_0 | \bx_t, E_{\bphi}(\bx_0)] - \bx_0\rangle\right] \\
            = & \E_{\bx_0,\bx_t, E_{\bphi}(\bx_0)}\left[ \langle\bx_{\btheta}(\bx_t, t, E_{\bphi}(\bx_0)) - \E [\bx_0 | \bx_t, E_{\bphi}(\bx_0)],\E [\bx_0 | \bx_t, E_{\bphi}(\bx_0)] - \bx_0\rangle\right] \\
            = & \E_{\bx_t, E_{\bphi}(\bx_0)}\E_{\bx_0|\bx_t, E_{\bphi}(\bx_0)} \left[ \langle\bx_{\btheta}(\bx_t, t, E_{\bphi}(\bx_0)) - \E [\bx_0 | \bx_t, E_{\bphi}(\bx_0)],\E [\bx_0 | \bx_t, E_{\bphi}(\bx_0)] - \bx_0\rangle\right] \\
            = & \E_{\bx_t, E_{\bphi}(\bx_0)}\left[\langle \bx_{\btheta}(\bx_t, t, E_{\bphi}(\bx_0)) - \E [\bx_0 | \bx_t, E_{\bphi}(\bx_0)],\E_{\bx_0|\bx_t, E_{\bphi}(\bx_0)} \left[ \E [\bx_0 | \bx_t, E_{\bphi}(\bx_0)] - \bx_0\right] \rangle \right].\\
        \end{aligned}
    \end{equation}
    Due to the property of conditional expectation, we have that
    \begin{equation}
        \begin{aligned}
            \E_{\bx_0|\bx_t, E_{\bphi}(\bx_0)} \left[ \E [\bx_0 | \bx_t, E_{\bphi}(\bx_0)] - \bx_0\right]=\E [\bx_0 | \bx_t, E_{\bphi}(\bx_0)] - \E [\bx_0 | \bx_t, E_{\bphi}(\bx_0)]=0.
        \end{aligned}
    \end{equation}
    Thus we have
    \begin{equation}
        \E_{\bx_0,\bx_t}\left[ \langle\bx_{\btheta}(\bx_t, t, E_{\bphi}(\bx_0)) - \E [\bx_0 | \bx_t, E_{\bphi}(\bx_0)],\E [\bx_0 | \bx_t, E_{\bphi}(\bx_0)] - \bx_0\rangle\right] = 0.  
    \end{equation}
    Thus
    \begin{equation}
        \begin{aligned}
            &\argmin\limits_{\bx_{\theta}} \Loss_{\bx_0, DSM, \bphi}\\
            =& \argmin\limits_{\bx_{\theta}}\E_t\bigl\{ \lambda(t)\E_{\bx_0,\bx_t}\left[\|
            \bx_{\btheta}(\bx_t, t, E_{\bphi}(\bx_0)) - \E [\bx_0 | \bx_t, E_{\bphi}(\bx_0)]\|^2\right] \\
            & + 2\lambda(t)\E_{\bx_0,\bx_t}\left[\langle\bx_{\btheta}(\bx_t, t, E_{\bphi}(\bx_0)) - \E [\bx_0 | \bx_t, E_{\bphi}(\bx_0)],\E [\bx_0 | \bx_t, E_{\bphi}(\bx_0)] - \bx_0\rangle \right]\bigl\} \\
            =& \argmin\limits_{\bx_{\theta}}\E_t\bigl\{ \lambda(t)\E_{\bx_0,\bx_t}\left[\|
            \bx_{\btheta}(\bx_t, t, E_{\bphi}(\bx_0)) - \E [\bx_0 | \bx_t, E_{\bphi}(\bx_0)]\|^2\right] \\
            =& \E [\bx_0 | \bx_t, E_{\bphi}(\bx_0)].
        \end{aligned}
    \end{equation}
    Substitute the minimizer of $\Loss_{\bx_0, DSM}$ into it, we get the minimum of $\Loss_{\bx_0, DSM}$
    \begin{equation}
        \begin{aligned}
            &\min_{\bx_{\theta}} \Loss_{\bx_0, DSM, \bphi}\\
            =& \min_{\bx_{\theta}} \E_t\left\{\lambda(t) \E_{\bx_0} \E_{\bx_t | \bx_0} \left[ \|\bx_{\btheta}(\bx_t, t, E_{\bphi}(\bx_0)) - \bx_0 \|^2 \right] \right\} \\ 
            =& \E_t\left\{\lambda(t) \E_{\bx_0} \E_{\bx_t | \bx_0} \left[ \|
            \E [\bx_0 | \bx_t, E_{\bphi}(\bx_0)] - \bx_0 \|^2 \right] \right\}\\
            =& \E_t\left\{\lambda(t)\E_{\bx_t, E_{\bphi}(\bx_0)}\E_{\bx_0|\bx_t, E_{\bphi}(\bx_0)} \left[ 
            (\E [\bx_0 | \bx_t, E_{\bphi}(\bx_0)] - \bx_0)^T(\E [\bx_0 | \bx_t, E_{\bphi}(\bx_0)] - \bx_0)  \right]\right\}\\
            =& \E_t\left\{\lambda(t)\E_{\bx_t, E_{\bphi}(\bx_0)}\E_{\bx_0|\bx_t, E_{\bphi}(\bx_0)} \left[ 
            \Tr((\E [\bx_0 | \bx_t, E_{\bphi}(\bx_0)] - \bx_0)^T(\E [\bx_0 | \bx_t, E_{\bphi}(\bx_0)] - \bx_0))  \right]\right\}\\
            =& \E_t\left\{\lambda(t)\E_{\bx_t, E_{\bphi}(\bx_0)}\E_{\bx_0| \bx_t, E_{\bphi}(\bx_0)} \left[ 
            \Tr((\E [\bx_0 | \bx_t, E_{\bphi}(\bx_0)] - \bx_0)(\E [\bx_0 | \bx_t, E_{\bphi}(\bx_0)] - \bx_0)^T)  \right]\right\}\\
            =& \E_t\left\{\lambda(t)\E_{\bx_t, E_{\bphi}(\bx_0)}\left[\Tr(\E_{\bx_0|\bx_t, E_{\bphi}(\bx_0)} \left[ 
            (\E [\bx_0 | \bx_t, E_{\bphi}(\bx_0)] - \bx_0)(\E [\bx_0 | \bx_t, E_{\bphi}(\bx_0)] - \bx_0)^T  \right])\right]\right\}\\
            =& \E_t\left\{\lambda(t)\E_{\bx_t, E_{\bphi}(\bx_0)}\left[\Tr(\Cov[\bx_0 | \bx_t, E_{\bphi}(\bx_0)])\right]\right\} \\
            =& \E_t\left\{\lambda(t)\E_{\bx_0, \bx_t}\left[\Tr(\Cov[\bx_0 | \bx_t, E_{\bphi}(\bx_0)])\right]\right\} \geq 0.
        \end{aligned}
    \end{equation}
\end{proof}
\subsection{Proof of Lemmas}
\setcounter{lemma}{0}
\begin{lemma}
    $\bU$ and $\bV$ are two square-integrable random variables. $\bU$ is $\mathcal{G}$-measurable and $\E\left[\bV|\mathcal{G}\right] = \mathbf{0}$, then
    \begin{equation}
        \E \left[ \| \bU+\bV \|^2 \right] = \E \left[ \| \bU\|^2 \right] + \E \left[ \| \bV \|^2 \right].
    \end{equation}
\end{lemma}
\begin{proof}
    \begin{equation}
        \begin{aligned}
            & \E \left[ \| \bU+\bV \|^2 \right] \\
            =& \E \left[ \| \bU\|^2 \right] + \E \left[ \| \bV \|^2 \right] + 2\E \left[ \langle \bU,\bV\rangle\right],
        \end{aligned}
    \end{equation}
    while
    \begin{equation}
    \E \left[ \langle \bU,\bV\rangle\right]
            = \E\left[\E \left[ \langle \bU,\bV\rangle | \mathcal{G}\right]\right] 
            = \E\left[\langle \bU ,\E \left[ \bV | \mathcal{G}\right] \rangle\right] 
            = 0.
    \end{equation}
\end{proof}
\begin{lemma}
    $\bX$ is a random variable, $\mathcal{F}$ and $\mathcal{G}$ are two $\sigma$-algebras such that $\mathcal{G} \subset \mathcal{F}$, then we have 
    \begin{equation}
        \E \left[ \|\E \left[\bX |\mathcal{F} \right]\|^2 \right] \geq \E \left[ \|\E \left[\bX |\mathcal{G} \right]\|^2 \right].
    \end{equation}
\end{lemma}
\begin{proof}
    Let $\bU = \E \left[\bX |\mathcal{G} \right]$ and $\bV = \E \left[\bX |\mathcal{F} \right] - \E \left[\bX |\mathcal{G} \right]$, $\bU$ is $\mathcal{G}$-measurable and according to the tower property of conditional expectation
    \begin{equation}
    \E\left[\bV| \mathcal{G}\right ] 
            = \E\left[ \E \left[\bX |\mathcal{F} \right]| \mathcal{G}\right] - \E \left[\bX |\mathcal{G} \right]
            =\E \left[\bX |\mathcal{G} \right] - \E \left[\bX |\mathcal{G} \right]
            = 0.
    \end{equation}
    According to lemma 1, we have
    \begin{equation}
        \E \left[ \|\E \left[\bX |\mathcal{F} \right]\|^2 \right] = \E \left[ \|\E \left[\bX |\mathcal{G} \right]\|^2 \right] + \E \left[ \|\E \left[\bX |\mathcal{F} \right] - \E \left[\bX |\mathcal{G} \right]\|^2 \right] \geq \E \left[ \|\E \left[\bX |\mathcal{G} \right]\|^2 \right].
    \end{equation}
\end{proof}

\begin{lemma}
    Let $\Pi_t$ be the set of distribution $p(x)$ on $\mathbb{R}^n$ satisfying the following condition:
    \begin{equation}
        \E_p\left[\mathbf{X}\right] = \mathbf{0}, \quad\Tr\left(\Cov_p\left[\mathbf{X}\right]\right) = t.
    \end{equation}
    Then the n-dimensional Gaussian distribution with mean $\mathbf{0}$ and covariance matrix $\Sigma = \frac{t}{n}I_n$ is the maximum entropy distribution in $\Pi_t$
\end{lemma}

\begin{proof}
    We know that any probability distribution on $\mathbb{R}_n$ with finite means and finite covariances has its entropy bounded by the entropy of the n-dimensional Gaussian with the same means and covariances. Thus the maximum entropy distribution in $\mathbb{R}_n$ lies among the n-dimensional Gaussians in $\Pi_t$, which are the distributions of the form
    \begin{equation}
        p_{\Sigma}(\bx) = \frac{1}{\sqrt{(2\pi)^n \det (\Sigma)}}\exp{\left(-\frac{\bx^T \Sigma^{-1}\bx}{2}\right)},
    \end{equation}
    where $\Sigma$ is a positive-definite symmetric matrix with trace $t$. The entropy of $p_{\Sigma}$ is
    \begin{equation}
        h(p_{\Sigma}) = \frac{1}{2}\left(n + \log{\left((2\pi)^n \det (\Sigma)\right)}\right).
    \end{equation}
    The arithmetic-geometric mean inequality on the eigenvalues of $\Sigma$ derives
    \begin{equation}
        \frac{1}{n}\Tr(\Sigma) \geq \sqrt[n]{\det (\Sigma)}.
    \end{equation}
    The equality holds if and only if all the eigenvalues of $\Sigma$ are equal. Therefore
    \begin{equation}
        h(p_{\Sigma}) \leq \frac{n}{2}\left(1 + \log{\left( \frac{2\pi t}{n}\right)}\right).
    \end{equation}
    Thus the n-dimensional Gaussians with mean $\mathbf{0}$ and covariance $\frac{t}{n}I_n$ is the maximum entropy distribution in $\Pi_t$.
\end{proof}

\clearpage
\section{Hyper-paramter Configurations}
\label{app:hyper}

\begin{table*}[h]
    \centering
    \caption{Hyper-parameter configurations for node classification datasets.
    }
    \begin{adjustbox}{max width=\textwidth} 
    \renewcommand{\arraystretch}{1.05}
    \begin{tabular}{c|c|cccccc}
        \toprule[1.2pt]
            & Dataset &   Cora      & CiteSeer      & PubMed                & Ogbn-arxiv        & Computer               & Photo       \\
         \midrule
        \multirow{8}{*}{Hyper-parameters} 
        & feat\_drop            & 0.3   & 0.4   &  0.2  &  0.1  &  0.4  &  0.1   \\
        & att\_drop             &  0.1  &  0.2  &  0.2  &  0.2  &  0.2  & 0.3    \\
        & num\_head             &  4  &  4  &  2  & 2   &  2  &  4   \\
        & num\_hidden           &  1024  &  1024  &  1024  &  256  &  512  & 512    \\
         & learning\_rate       &  1e-4  &  1e-4  &  1e-4  &  1e-3  &  1e-4  &  3e-4   \\
         & mask\_ratio          &  0.7  & 0.7   &  0.7  &  0.7  &  0.7  &  0.7   \\
         & noise\_schedule      &  sigmoid  &  sigmoid  &  sigmoid  &  inverted  &   quad &   sigmoid  \\
         & optimizer             &  Adam &  Adam  &  Adam  &  Adam  &  Adam  & Adam    \\
        \bottomrule[1.2pt]
    \end{tabular}
    \end{adjustbox}
\end{table*}

\begin{table*}[h]
    \centering
    \caption{Hyper-parameter configurations for graph classification datasets.
    }
    \begin{adjustbox}{max width=\textwidth} 
    \renewcommand{\arraystretch}{1.05}
    \begin{tabular}{c|c|ccccc}
        \toprule[1.2pt]
            & Dataset &  IMDB-B     & IMDB-M     & PROTEINS   & COLLAB     & MUTAG       \\
         \midrule
        \multirow{8}{*}{Hyper-parameters} 
        & feat\_drop            & 0.3   & 0.3   &  0.3  &  0.3  &  0.3   \\
        & att\_drop             &  0.1  &  0.2  &  0.2  &  0.2  &  0.2      \\
        & num\_head             &  2  &  2  &  2  & 2   &  2   \\
        & num\_hidden           &  512  &  512  &  512  &  512  &  32    \\
         & learning\_rate       &  1e-4  &  1e-4  &  1e-4  &  1e-3  &  1e-4  \\
         & mask\_ratio          &  0.3  & 0.3   &  0.3  &  0.3  &  0   \\
         & noise\_schedule      &  sigmoid  &  sigmoid  &  sigmoid  &  sigmoid  &   sigmoid  \\
         & optimizer             &  Adam &  Adam  &  Adam  &  Adam  &  Adam   \\
        \bottomrule[1.2pt]
    \end{tabular}
    \end{adjustbox}
\end{table*}

}
\clearpage